\documentclass{article}

\usepackage[final]{neurips_2024}

\usepackage[utf8]{inputenc} 
\usepackage[T1]{fontenc}    
\usepackage{hyperref}       
\usepackage{url}            
\usepackage{booktabs}       
\usepackage{amsfonts}       
\usepackage{nicefrac}       
\usepackage{microtype}      
\usepackage{xcolor}         

\usepackage{subcaption}
\usepackage{wrapfig}
\usepackage{algorithmic}
\usepackage{algorithm}
\usepackage{graphicx}
\usepackage{amsmath}
\usepackage{amsthm}
\usepackage{amssymb}
\usepackage{amsfonts}
\usepackage{stmaryrd}
\usepackage[shortlabels]{enumitem}
\usepackage{mathabx}
\usepackage{multirow}
\usepackage{thmtools}
\usepackage{thm-restate}
\usepackage{hyperref}
\usepackage{cleveref}

\declaretheorem[name=Lemma]{lemma}
\usepackage{rotating}

\title{Belief-State Query Policies for User-Aligned POMDPs}

\author{%
  \textbf{Daniel Bramblett} {\normalfont and} \textbf{Siddharth Srivastava} \\
  Autonomous Agents and Intelligent Robots Lab\\
  School of Computing and Augmented Intelligence\\
  Arizona State University, AZ, USA \\
  \texttt{\{drbrambl,siddharths\}@asu.edu} \\
}

\begin{document}

\maketitle

\begin{abstract}
Planning in real-world settings often entails addressing partial observability while aligning with users' requirements. We present a novel framework for expressing users' constraints and preferences about agent behavior in a partially observable setting using parameterized belief-state query (BSQ) policies in the setting of goal-oriented partially observable Markov decision processes (gPOMDPs). We present the first formal analysis of such constraints and prove that while the expected cost function of a parameterized BSQ policy w.r.t its parameters is not convex, it is \emph{piecewise constant} and yields an implicit \emph{discrete parameter search space} that is finite for finite horizons. This theoretical result leads to novel algorithms that optimize gPOMDP agent behavior with guaranteed user alignment. Analysis proves that our algorithms converge to the optimal user-aligned behavior in the limit. Empirical results show that parameterized BSQ policies provide a computationally feasible approach for user-aligned planning in partially observable settings.
\end{abstract}

\section{Introduction}
\label{sec:introduction}
Users of sequential decision-making (SDM) agents in partially observable settings often have requirements and preferences on expected behavior, ranging from safety concerns to 
high-level knowledge of task completion requirements. 
However, users are ill-equipped to specify desired behaviors from such agents. For instance, although
reward engineering can often encode fully observable preferences \citep{NEURIPS2021_a7f0d2b9,NEURIPS2023_a5357781}, it requires 
significant trial-and-error, and can produce unintended behavior even when done by experts 
working on simple domains
~\citep{Booth_Knox_Shah_Niekum_Stone_Allievi_2023}. These challenges are compounded in partially observable environments, where the agent will not know the full state on which the users' requirements and preferences are typically defined. For example,  defining a reward function on the belief state to align the agent's behavior with the user
can result in wireheading \citep{everitt2016avoiding} (see Sec.\,\ref{sec:related} for further discussion on related work). 

Consider a simplified, minimal example designed to  illustrate the key principles (Fig.\,\ref{fig:spaceship_repair_combined}(a)).
A robot located on a spaceship experiences a
communication error with the ship and needs to decide whether to
attempt to  repair  itself or the ship. Importantly, while 
a robot error is harder to detect, the user would rather risk repairing the robot
than repairing the ship, as each repair risks introducing additional failures. In other words, the user may expect the robot to work with the following goals and preferences:
\emph{The objective is to fix the communication channel. First, if there is a ``high'' likelihood that the robot is broken, it should try to repair itself;
otherwise, if there is a ``high'' likelihood that the ship is broken, it should try to repair that.} Such
preferences go beyond preferences in fully observable settings: they use
queries on the current belief state for expressing users' requirements while using the conventional paradigm of stating objectives in terms of the true underlying state. Such a formulation avoids wireheading, allowing users to express their constraints and preferences in partially observable settings.  Although
such constraints on behavior are intuitive and common, they leave a significant
amount of uncertainty to be resolved by the agent: it needs to optimize the threshold values of ``high'' probability  under which each rule would apply while attempting to achieve the objective.


\begin{figure*}
    \centering
    \includegraphics[width=1.0\textwidth]{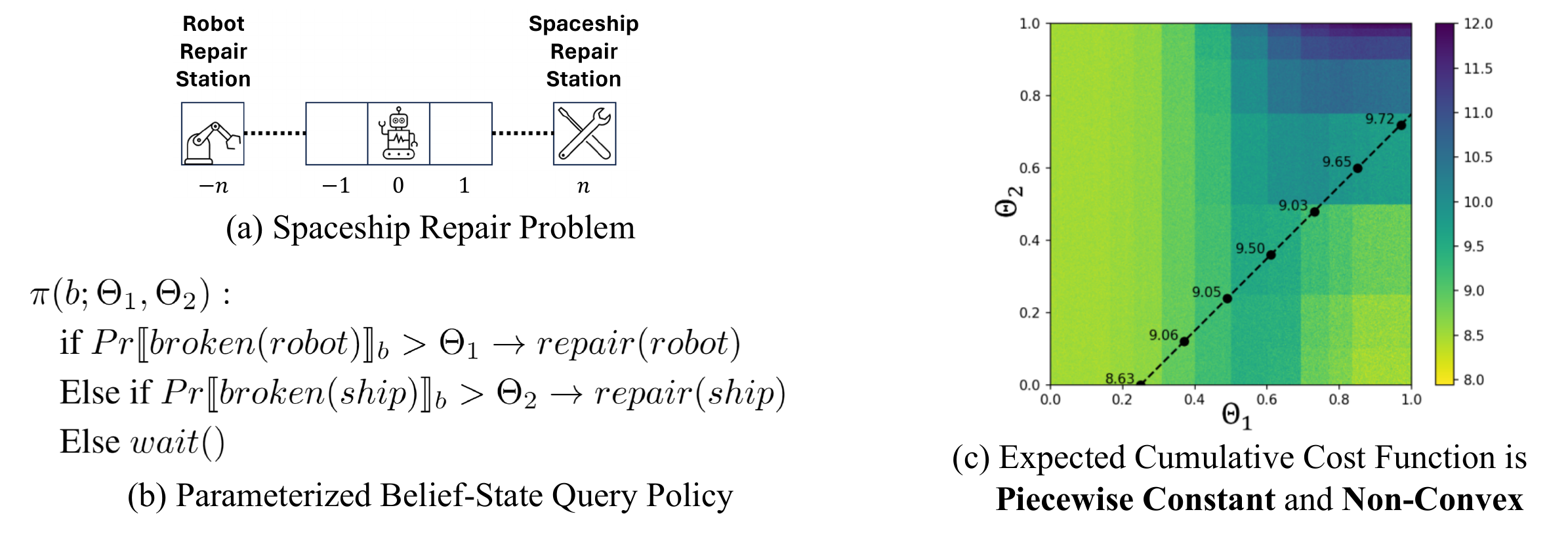}

    \caption{(a) Spaceship Repair running example. (b) parameterized BSQ policy for the user preference from the Introduction. (c) The expected cumulative cost function for (b) with a horizon of 12.}
    \label{fig:spaceship_repair_combined}
\end{figure*}

We introduce mathematical and algorithmic foundations for addressing these problems by defining constraints on behaviors in terms of properties of the belief state, expressed through belief-state queries (BSQs). We prove the surprising result that although the space of possible threshold values in preferences such as the one listed above is uncountably infinite, only a finite number of evaluations are required for computing optimal, user-aligned policies for finite-horizon problems. We use this result to develop a probabilistically complete algorithm for computing optimal constrained policies. Our main contributions are: 
\begin{enumerate}
    \item A framework for encoding user requirements and preferences over agent behavior in goal-oriented partially observable Markov decision processes (Sec.\,\ref{sec:formal}).
    \item Mathematical analysis proving that the expected cost function of a parameterized BSQ policy w.r.t its parameters is piecewise constant but generally non-convex. (Sec.\,\ref{sec:theory}).
    \item A probabilistically complete algorithm for computing optimal user-aligned policies in goal-oriented POMDPs (Sec.\,\ref{sec:algorithms}).
    \item Empirical evaluation on a diverse set of problems showing both the efficiency of our algorithm and the quality of the computed user-aligned policies.
    (Sec.\,\ref{sec:empirical}).
\end{enumerate}

\section{Related Work}
\label{sec:related}


Planning over
preferences has been well studied in fully observable settings
~\citep{baier2007,aguas2016generalized}. 
\citet{NEURIPS2022_70b8505a} present an approach  for
complying with an LTL specification while carrying out reinforcement learning. Other approaches for using LTL specifications use the grounded state to create a reward function to teach reinforcement learning agents \citep{toro2018teaching,pmlr-v139-vaezipoor21a}. These approaches do not extend to partially observable settings as they consider agents that can access the complete state. 

In partially observable settings, existing approaches for using domain knowledge and preferences require extensive, error-prone reward design and/or do not guarantee compliance. LTL specifications have been incorporated either by designing a reward function that incentivizes actions more likely to adhere to these specifications \citep{liu2021leveraging,tuli2022learning} or by imposing a compliance threshold \citep{ahmadi2020stochastic}. In both approaches, the user calibrates rewards for user alignment with those for objective completion; it is difficult to ensure user alignment. We focus on the problem of guaranteeing user alignment without  reward engineering.

\citet{mazzi2021rule,mazzi2023learning} proposed expressing domain control knowledge using belief state probabilities. \citet{mazzi2021rule} used expert-provided rule templates and execution traces to construct a shield to prevent irregular actions. \citet{mazzi2023learning} used execution traces and domain-specified belief-state queries to learn action preconditions over the belief state. Both approaches use input traces and focus on ensuring a policy is consistent with previously observed behavior. We address the complementary problem of computing user-aligned policies without past traces.

Belief-state queries  have been used to solve POMDPs with uniform parameter
sampling~\citep{Srivastava2014FirstOrderOP} but formal analysis, feasibility of
optimizing BSQ policies, and the existence of provably convergent algorithms
have remained open as research questions prior to this work.

\section{Formal Framework}
\label{sec:formal}

This section formally defines the BSQ framework, which expresses user requirements on an agent’s belief and is designed for relational goal-oriented partially observable Markov decision processes. 

\subsection{Goal-Oriented Partially Observable Markov Decision Process}
\label{sec:gPOMDP}

Partially observable Markov decision processes (POMDPs) constitute a standard mathematical framework for modeling SDM problems in partially observable, stochastic settings \citep{kaelbling1998planning,smallwood1973optimal}. 
State-of-the-art POMDP solvers often rely on approximate online approaches \citep{NIPS2010_edfbe1af,somani2013despot} where recent work addresses the problem of 
obtaining performance bounds \citep{NEURIPS2023_fc6bd0ee,lim2023optimality}. 

We use goal-oriented POMDPs (gPOMDPs), where the agent aims to complete one of the tasks/goals. This eliminates the burden of error-prone reward engineering  by using a default cost function that associates a constant cost for each timestep before reaching the goal. 
E.g., the Spaceship Repair problem (Sec.\,\ref{sec:introduction}) has two objects: the robot and the spaceship. A state is defined using a Boolean function $broken(o)$ representing whether object $o$ needs repair and an integer-valued function $rlocation()$ representing the robot's location. 
Both functions are not observable. 
The agent has two types of actions: try to repair object $o$ ($repair(o))$ or wait $(wait())$. A transition function expresses the distribution of $rlocation()$ 
depending on the action taken and the robot's previous location.
At each timestep, the robot receives a noisy observation $obs\_err(o)$ regarding the status of  object $o$. The set of observations can be expressed as $\{obs\_err(robot),obs\_err(ship)\}$.
Due to noisy perception, $obs\_error(o)$ may not match $broken(o)$. An observation function denotes the probability of each observation conditioned on the (hidden) current state. 
The goal is to reach the repair station corresponding to the truly broken component. We define gPOMDPs formally as follows.

\begin{restatable}{definition}{gpomdpdef}
A \emph{goal-oriented partially observable Markov decision process} $\mathcal{P}$ is defined as $\langle \mathcal{C}, \mathcal{F}, \mathcal{A}, \mathcal{O}, \mathcal{T},  \Omega, \mathcal{G}, \emph{Cost},  H,b_0 \rangle$ where $\mathcal{C}$ is the finite set of constant symbols and $\mathcal{F}$ is the finite set of functions. The set of 
state variables for $\mathcal{F}$, $\mathcal{V}_F$, is defined as all instantiations of functions in $\mathcal{F}$ 
with objects in $O$.
The set of states $\mathcal{S}$ is the set of all possible valuations for $\mathcal{V}_F$; $\mathcal{A}$ is a finite set of actions, $\mathcal{O}$ is a subset of $\mathcal{F}$ of observation predicates, $\mathcal{T}: \mathcal{S} \times \mathcal{A} \times \mathcal{S} \rightarrow [0, 1]$ is the transition function $T(s,a,s')=Pr(s'|a,s)$; $\mathcal{G} \subseteq \mathcal{S}$ is the set of goal states that are also sink states, $\Omega : \mathcal{S} \times \mathcal{A} \times \mathcal{O} \rightarrow [0, 1]$ is the observation function; $\Omega(s,a,o)=Pr(o|s,a)$, $\emph{Cost}(s)=\{0$ if $s\in\mathcal{G}$;else $1\}$ is the cost function, $H$ is the horizon, and $b_0$ is the initial belief state. A \emph{solution for a gPOMDP} is a policy that has a non-zero probability of reaching $\mathcal{G}$ in $H-1$ timesteps.
\label{def:gpomdp}
\end{restatable}

\subsection{Belief-State Queries and Policies}
\label{sec:bsq}

Computing a policy for any gPOMDP requires planning around state uncertainty. This is done  using the concept of a \emph{belief state}, which is a probability distribution over the currently possible states. Formally, the belief state constitutes a 
 sufficient statistic for observation-action histories \citep{astrom1965optimal}. We express user requirements using queries on the current belief state.
 
For any belief state $b$, when action $a$ is taken and observation $o$ is observed, the updated belief state is computed using $b^\prime(s^\prime)=\alpha \Omega(s^\prime,a,o)\sum_{s}\mathcal{T}(s,a,s^\prime)b(s)$ where $\alpha$ is the normalization factor. We refer to this belief propagation  as $b^\prime = bp(b,a,o)$. We extend the notation to refer to the sequential application of this equation to arbitrary bounded histories as  $bp^*(b_0,a_1,o_1,...,a_n,o_n) = bp(\ldots bp(bp(b_0, a_1, o_1), a_2, o_2)\ldots)$.

For example, the Spaceship Repair problem user preference  has the expression \emph{``a high likelihood that the robot is broken''}. This can be expressed as a query on 
a belief state $b$: $Pr \llbracket broken(robot)\rrbracket_b > \Theta_{rob}$ where $\Theta_{rob}$ is a parameter. If $rlocation()$ is fully observable, the expression \emph{``the robot  location is smaller than $\Theta_l$ in a belief state $b$''} can be expressed as $Pr \llbracket rlocation() < \Theta_l \rrbracket_b == 1$. 
We can combine both queries to express \emph{``a high likelihood the robot is broken and its location is lower than $\Theta_l$''}, as: $Pr \llbracket broken(robot)\rrbracket_b > \Theta_{rob} \wedge Pr \llbracket rlocation() < \Theta_l \rrbracket_b == 1$.

Formally, BSQs use the vocabulary of the underlying gPOMDP. There are two types of queries we can ask: (1) whether formula $\varphi$ is true with a probability that satisfies a threshold $\Theta$; (2) whether the fully observable portion of the state satisfies a formula $\varphi$ containing a threshold $\Theta$. These thresholds represent the parameters of a parameterized BSQ policy. The agent must optimize these parameters to achieve the goal while aligning with the user's requirements.
BSQs can be combined using conjunctions or disjunctions to express more complex requirements, which we define as a compound BSQ in Def.\, \ref{def:comp_bsq}. We omit subscripts when clear from context.

\begin{restatable}{definition}{bsqdef}
    A \emph{belief-state query} $\lambda_\mathcal{P}(b;\varphi,\circ,\Theta)$, where $b$ is a belief state, $\varphi$ is a first-order logic formula composed of functions in gPOMDP $\mathcal{P}$, $\circ$ is any comparison operator, and $\Theta\in\mathbb{R}$ is a parameter, is defined as $\lambda_\mathcal{P}(b;\varphi,\circ,\Theta)=Pr \llbracket\varphi\rrbracket_b \circ \Theta$.
    \label{def:bsq}
\end{restatable}


\begin{restatable}{definition}{compbsqdef}
    \label{def:comp_bsq}
    A \emph{compound BSQ} $\Psi(b;\overline{\Theta})$, where b is a belief state and $\overline{\Theta}\in\mathbb{R}^n$ 
    , is either a conjunction or a disjunction of BSQs that contain $n$ total parameters.
\end{restatable}

\label{sec:bsq_pref}

We use BSQs to formally express user requirements of the form discussed in the introduction by mapping BSQs with variable parameters to actions.
Fig.\,\ref{fig:spaceship_repair_combined}(b) illustrates this with a parameterized BSQ policy for the Spaceship Repair problem. Formally,



\begin{restatable}{definition}{bsqprefdef}
Let $b$ be a belief state and $\overline{\Theta}$ be  a tuple of $n$ parameter variables over $\mathbb{R}$. An n-parameter \emph{Parameterized Belief-State Query policy} $\pi(b,\overline{\Theta})$ is a tuple of rules $\{r_1,...,r_m\}$ where each $r_i= \Psi_i\rightarrow a_i$ is composed of a compound BSQ $\Psi_i$ and an action $a_i\in\mathcal{A}$. The set  $\{\Psi_1,...,\Psi_m\}$ is mutually exclusive and covers the n-dimensional parameter space $\mathbb{R}^n$.
\label{def:bsq_pref}
\end{restatable}

In practice, mutually exclusive coverage is easily achieved using an \emph{if... then...  else} structure, where each condition includes a conjunction of the negation of preceding conditions and the list of rules includes a terminal \emph{else} with the catchall BSQ \emph{True} (Fig.\,\ref{fig:spaceship_repair_combined}).
Any assignment of values $\overline{\vartheta}\in\mathbb{R}^n$ to the parameters $\overline{\Theta}$ of a parameterized BSQ policy produces an executable policy that maps every possible belief state to an action:

\begin{restatable}{definition}{bsqpolicydef}
A \emph{BSQ policy} $\pi(b,\overline{\vartheta})$ is a parameterized BSQ policy $\pi(b,\overline{\Theta})$ with an assignment in $\mathbb{R}$ to each of the $n$ parameters $\overline{\Theta}$.
\label{def:bsq_policy}
\end{restatable}

Let $Pr^\pi_t(\mathcal{G})$ be the probability that an execution of a policy $\pi$ reaches a state in $\mathcal{G}$ within $t$ timesteps. A BSQ policy $\pi(b,\overline{\vartheta})$ is said to be \emph{a solution to a gPOMDP} with goal $\mathcal{G}$ and horizon $H$ iff $Pr^{\pi(b,\overline{\vartheta})}_{H-1}(\mathcal{G})>0$. The quality of a BSQ policy is defined as its expected cost;  due to the uniform cost function in the definition of gPOMDPs, the expected cost of a BSQ policy is the expected time taken to reach a goal state. Formally, the \emph{expected cost of a BSQ} policy $\pi(b,\overline{\vartheta})$ is  $E_\pi(\overline{\vartheta};H) = \sum_{t=1}^H t\times Pr_{\mathcal{G},t}[\pi(b,\overline{\vartheta})]$, where $H$ is the horizon and $Pr_{\mathcal{G},t}[\pi(b,\overline{\vartheta})]$ is the probability of policy $\pi(b,\overline{\vartheta})$ reaching a goal state for the first time at timestep $t$.
Thus, given a gPOMDP $\mathcal{P}$, with goal $\mathcal{G}$, and a parameterized BSQ policy $\pi(b,\overline{\Theta})$, the objective is to  compute:

\begin{equation*}
    \overline{\vartheta}^* = \emph{argmin}_{\overline{\vartheta}} \{ E_\pi(\overline{\vartheta};H): Pr^{\pi(b, \overline{\vartheta})}_{H-1}(\mathcal{G}) >0   \}
\end{equation*}


\section{Formal Analysis}
\label{sec:theory}
Our main theoretical result is that the continuous space of policy parameters is, in fact, partitioned into finitely many constant-valued convex sets. This insight allows the development of scalable algorithms for computing low-cost user-aligned policies.
We introduce formal concepts and key steps in proving this result here; complete   proofs  for all results are available in the Appendix. We begin with the notion of strategy trees to conceptualize the search process for BSQ policies.



\subsection{Strategy Trees}
\label{sec:strate_tree_features}

\begin{figure}[t]
    \centering
    \includegraphics[width=1\textwidth]{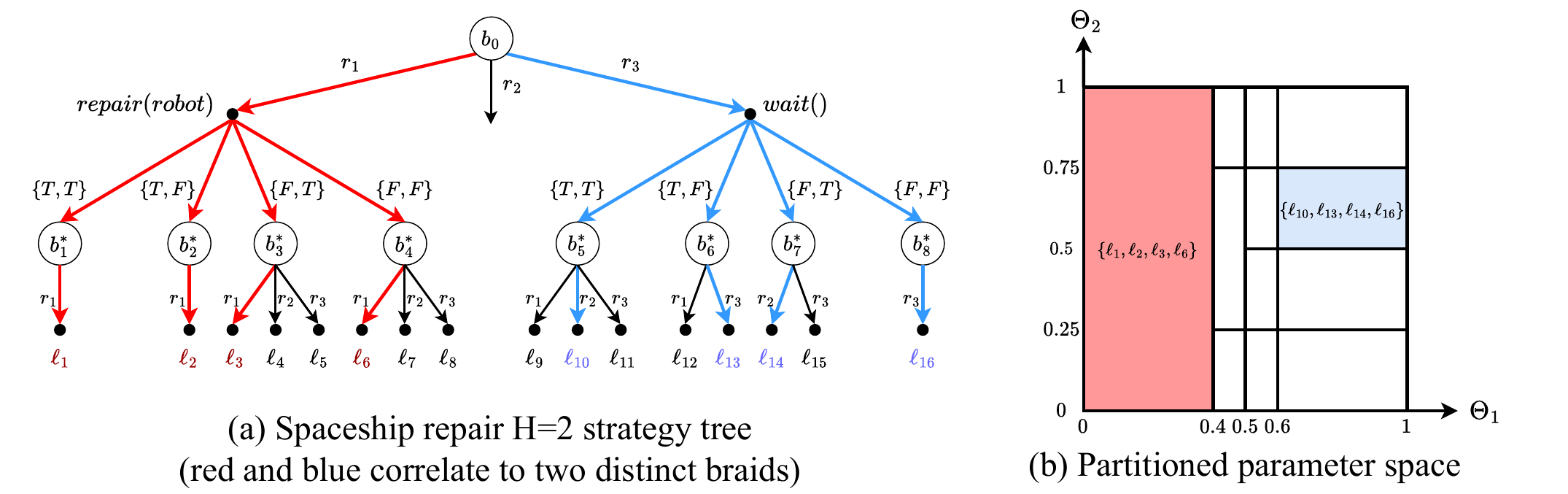}
    \caption{(a) Strategy tree created from parameterized BSQ policy in Fig.\,\ref{fig:spaceship_repair_combined} and Spaceship Repair gPOMDP with horizon of 2. (b) Complete partitions of parameter space with two of the braids highlighted. Error detection sensor accuracy for the robot and ship is 60\% and 75\%, respectively.}
    \label{fig:game_tree_visual}
\end{figure}


Every parameterized BSQ policy $\pi(b,\overline{\Theta})$ and gPOMDP $\mathcal{P}$ defines a strategy tree (e.g., Fig.\,\ref{fig:game_tree_visual}(a)) that captures the possible decisions at each execution step. Intuitively, the tree starts at a belief node representing the initial belief state. Outgoing edges from belief nodes represent rule selection in $\pi(b,\overline{\Theta})$, resulting in action nodes. Outgoing edges from action nodes represent possible observations, leading to belief nodes representing the corresponding updated belief. If the tree is truncated at horizon $H$,  each leaf represents the outcome of a unique trajectory of rules from $\pi(b,\overline{\Theta})$  and observations.

Each belief node represents a  belief state that can be calculated using the rule-observation trajectory  leading to that node. A labeling function $l:V_B \cup V_A \rightarrow B\cup A $ maps the set of belief nodes  $V_B$ to belief states in $B$ and the set of action nodes $V_A$ to actions in $\mathcal{A}$. For ease of notation we define  $b^*_i=l(v_i)$ for all belief nodes $v_i\in V_B$ and $a^*_j=l(v_j)$ for all action nodes $v_j\in V_A$. 


\begin{restatable}{definition}{strattreedef}
Let $\mathcal{P}$ be a gPOMDP,  $\pi(b,\overline{\Theta})$ be a parameterized BSQ policy for $\mathcal{P}$, and  $b_0$ be the initial belief state. The \emph{strategy tree} $\mathcal{T}_\pi(b_o)$  is defined as $\mathcal{T}_\pi(b_o) = \langle V, E \rangle$ where set $V = V_B \cup V_A$ contains belief nodes $V_B$ and action nodes $V_A$, whereas, set $E = E_B \cup E_A$ contains edges from belief nodes to action nodes ($E_b \subseteq V_B \times V_A$) and edges from action nodes to belief nodes ($E_A \subseteq V_A \times V_B$). $E_{B}$ is defined as $ \{ (v_i, r, v_j) | v_i\in V_B,v_j\in V_A, r \in \pi(b,\overline{\Theta}), \emph{ and } \exists \Psi: r=\Psi\rightarrow a^*_j \} $. 
 $E_a$ is defined as  $\{ (v_m, o, v_n) | v_m \in V_A,v_n \in V_B,o \in \mathcal{O},
 \exists (v_{p},r=\Psi\rightarrow a,v_m) \in E_b; b^*_n = \textrm{bp}(b^*_p, a,o) \}$. 
\label{def:strat_tree}
\end{restatable}

\paragraph{Non-convexity of the expected cost function} Each parameterized BSQ policy permits infinitely many BSQ policies, one for each assignment of real values to its parameters. Unfortunately, the expected cost of parameterized BSQ policies is not a convex function of these parameters. Fig.\,\ref{fig:spaceship_repair_combined}(c) shows this with a counterexample using the parameterized BSQ policy from Fig.\,\ref{fig:spaceship_repair_combined}(b), a horizon of 12, and setting the robot's initial distance from each repair station to 5. This plot was constructed by sampling the expected cost for 251,001 equally-spaced parameter assignments to the Fig.\,\ref{fig:spaceship_repair_combined}(a) parameterized BSQ policy.  $E_\pi(\overline{\vartheta}; H)$ is clearly not convex: the expected cost along the line $\Theta_2=\Theta_1-0.25$ has two inflection points at $\Theta_1 = 0.6$ and $\Theta_1 = 0.8$. This creates two local minima: $\Theta_1 \leq 0.16$ and $\Theta_1 \geq0.83\wedge\Theta_2\leq0.1$.  Intuitively, this is due to the short horizon, which causes the optimal strategy to be selecting a repair station and traversing to it regardless of the observations. 
This complicates finding good BSQ policies using existing solvers. However, every possible BSQ policy can be associated with a set of strategy tree leaves that are reachable under that policy. Thus, for a given horizon, there are only finitely many expected costs for BSQ policies for a given problem. 

The main challenge in computing good BSQ policies is that the set of possible BSQ policies with distinct expected costs grows exponentially with the horizon and good BSQ parameters could be distributed arbitrarily in the high-dimensional, continuous space of parameter values. We use strategy trees to define groups of leaves called braids, which we will then use to prove that the space of BSQ policy parameters turns out to be well-structured in terms of the expected cost function.

\paragraph{Braids} 
We refer to the set of all leaves reachable under a policy $\pi(b,\overline{\vartheta})$  as \emph{the braid of $\overline{\vartheta}$}.
Due to the mutual exclusivity of rules for every assignment of parameter values to a parameterized BSQ policy, at most, one outgoing edge can be taken from  each belief node (as these correspond to the rules and actions). However, the stochasticity of dynamics and observations allows for multiple outgoing edges to be possible from action nodes.   
E.g., in the strategy tree for the Spaceship Repair problem (Fig.\,\ref{fig:game_tree_visual}(a)),   leaves $\ell_{2}$ and $\ell_{10}$ cannot both be reachable under a BSQ policy because that would require rules $r_1$ and $r_2$ to be satisfied at the same belief. However, both $\ell_1$ and $\ell_5$ may be reachable under the same BSQ policy since their paths diverge on an action node. 
Formally,
\begin{restatable}{definition}{braiddef}
Let 
$H$ be the horizon, and let $\pi(b,\overline{\Theta})$ be a parameterized BSQ policy for a gPOMDP $\mathcal{P}$. The \emph{braid of a parameter assignment $\overline{\vartheta}$}, $braid_{\pi,H}(\overline{\vartheta})$, is the set of all leaves in strategy tree $\mathcal{T}_\pi(b_0)$ rooted at the initial belief $b_0$ that can be reached while executing $\pi(b, \overline{\vartheta})$: $braid_{\pi,H}(\overline{\vartheta})=\{\ell_H : \textrm{ the path to } \ell_H \textrm{ is } (r_1,o_1,...,r_H,o_H); \forall i\; r_i = \Psi_i \rightarrow a_i,\textrm{ }b_i =bp^*(b_0,r_1,o_1,...,r_i,o_i) \emph{ and }\overline{\vartheta}\textrm{  satisfies } \Psi_i$.
\label{def:braid}
\end{restatable}


The unique interval of parameter values where a leaf is reachable can be calculated by taking the intersection of the parameter intervals needed to satisfy each rule on the path to that leaf. This is because for any compound BSQ $\Psi$, we can compute the unique interval of parameter values $I(\Psi)$ under which $b$ will satisfy $I(\Psi)$ by substituting each BSQ in $\Psi$ with its corresponding inequality:
\begin{restatable}{lemma}{beliefparameval}
Let $\Psi(b;\overline{\Theta})$ be an n-dimensional compound BSQ.
There exists a set of intervals $I(\Psi) \subseteq \mathbb{R}^n$ s.t. 
$\Psi(b;\overline{\Theta})$ evaluates to true iff $\overline{\Theta}\in I(\Psi)$.
\label{lemma:belief_param_eval}
\end{restatable}
We can utilize this result to compute the unique interval of parameter values consistent with a braid by taking the intersection of the intervals of each leaf contained in that braid (Def.\,\ref{def:leaf_interval}):

\begin{restatable}{definition}{leafintervaldef}
Let $\pi(b,\overline{\Theta})$ be a parameterized BSQ policy, $\mathcal{P}$ be a gPOMDP, $b_0$ be the initial belief state, and $H$ be the horizon. The \emph{interval of leaf $\ell$}, $I(\ell)$, is defined as the intersection of intervals $\bigcap_i I(\Psi_i)$ of the conditions of each rule $r_i$ that occurs in the path to that leaf. 
The interval for a set of leaves $L$ is defined as $I(L)=\bigcap_{\ell_i \in L} I(\ell_i)$.
\label{def:leaf_interval}
\end{restatable}
Any leaf or braid with an empty parameter interval does not align with the user's requirements. For example, in Fig.\,\ref{fig:game_tree_visual}, note that $r_1$ is the only rule satisfiable if $r_1$ is selected from $b_0$ and the robot is observed to be broken. Using the Fig.\,\ref{fig:spaceship_repair_combined}(b) policy and assuming the sensor accuracy is 60\%, picking a rule other than $r_1$ implies that 50\% likelihood was high enough to fix the robot yet 60\% was not, which is a contradiction. Removing misaligned leaves and braids prunes the tree.


\subsection{BSQ Policies are Piecewise Constant }
\label{sec:piecewise_const}

We now use the concept of braids to prove that the continuous, high-dimensional space of parameter values of a parameterized BSQ policy reduces to a finite set of contiguous, convex partitions with each partition having a constant expected cost. This surprising result implies that although the expected cost of BSQ policies is not a convex function of parameter assignments, optimizing a parameterized BSQ policy requires optimization over a finite set rather than over a continuous space. We first define a notion of similarity over assignments to parameterized BSQ policies that define BSQ policies:




\begin{restatable}{definition}{equivdef}
Let $\pi(b, \overline{\Theta})$ be a parameterized BSQ policy, $\mathcal{P}$ be a gPOMDP, and $H$ be the horizon. Two assignments $ \overline{\vartheta}_1,\overline{\vartheta}_2 \in \overline{\Theta}$ are said to be similar, $\overline{\vartheta}_1 \equiv_H \overline{\vartheta}_2$, iff $braid_{\pi,H}(\overline{\vartheta}_1) = braid_{\pi,H}(\overline{\vartheta}_2)$.
\label{def:equiv}
\end{restatable}
It is trivial to show $\equiv_H$ is transitive, symmetric, and reflexive, making it an equivalence relation over $\mathbb{R}^n$. As such, $\equiv_H$ defines a partition over the same space:

\begin{restatable}{thm}{equivpartitions}
Let $\pi(b,\overline{\Theta})$ be a parameterized BSQ policy, $\mathcal{P}$ be a gPOMDP, $b_0$ be the initial belief state, and $H$ be the horizon. The operator $\equiv_H$ partitions $\mathbb{R}^n$.
\label{theorem:equiv_partitions}
\end{restatable}




However, this result is not sufficient to define the structure of partitions induced in 
$\mathbb{R}^n$, which will be required for an efficient optimization algorithm.
Based on Sec. \ref{sec:strate_tree_features}, we know that leaves whose trajectories diverge due to different rules must not be in the same braid. Furthermore, a belief state can only lead to one set of possible observations for an action regardless of the BSQ policy being followed. Intuitively, this prevents braids from being proper subsets of each other, which implies that the parameter intervals for two braids can never have overlapping parameter intervals. This gives us the desired structure for partitions induced in the space of parameter values for parameterized BSQ policies: there are parameter intervals corresponding to distinct braids in the policy tree. In other words, the set of braids partitions the parameter space into contiguous, high-dimensional intervals. This can be proved formally and stated as follows:

\begin{restatable}{thm}{disjointint}
Let $\pi(b,\overline{\Theta})$ be a parameterized BSQ policy, $b_0$ be the initial belief state, and $H$ be the horizon. Each partition $\rho$ created by operator $\equiv_H$ partitioning $\mathbb{R}^n$ is the disjoint intervals, $\rho \subseteq \mathbb{R}^n$ where $\forall \overline{\vartheta}  \in \rho$, $braid_{\pi,H}(\overline{\vartheta}) = L$ where $L$ is a fixed set of leaves.
\label{theorem:disjoint_int}
\end{restatable}
 Since each partition corresponds to a braid and each braid corresponds to a fixed set of leaves, which defines the expected cost for all policies corresponding to that braid, all policies defined by a partition of the parameter space have a constant expected cost.
As such, the domain of the expected cost function $E_\pi(\overline{\vartheta};H)$ for gPOMDP $\mathcal{P}$ can be represented as the disjoint intervals of each braid partition. Thus, $E_\pi(\overline{\vartheta};H)$ is piecewise constant. The following result formalizes this.



\begin{restatable}{thm}{piecewiseconstant}
Let $\pi(b,\overline{\Theta})$ be a parameterized BSQ policy, $\mathcal{P}$ be a gPOMDP, $b_0$ be the initial belief state, and $H$ be the horizon. Each partition created by $\equiv_H$ on $\mathbb{R}^n$ has a constant expected cost.
\label{theorem:piecewise_constant}
\end{restatable}
 In some situations,  the braids
that partition the parameter space can be calculated in closed
form (e.g., see the Appendix for partitions for the Spaceship Repair problem). The next section develops a general approach for computing the braids and intervals corresponding to a parameterized BSQ policy, for evaluating the expected cost for each such partition, and for optimizing over these partitions.

\section{Partition Refinement Search }
\label{sec:algorithms}

\begin{wrapfigure}{r}{0.48\textwidth}
\vspace{-25pt}
    \begin{minipage}{0.48\textwidth}
\centering
\begin{algorithm}[H]
\caption{Partition Refinement Search  (PRS)}
\label{alg:ips}
\begin{algorithmic}[1]
\STATE{Inputs: gPOMDP $\mathcal{P}$, parameterized BSQ policy $\pi(b,\overline{\Theta})$, horizon $H$}

\STATE{Output: Minimum cost partition and its expected cost $\langle\rho_{opt},\hat{E}[\rho_{opt}]\rangle$}

\STATE{$\rho_{init} \leftarrow \bigtimes_{\Theta\in\overline{\Theta}}\mathcal{D}_\Theta$}
\STATE{$X=\{\langle \rho_{init},\infty\rangle\},X_{opt}=\langle \rho_{init},\infty\rangle$}
\WHILE{$!TimeOut()$} 
\STATE{$\langle\rho,\hat{E}[\rho]\rangle \leftarrow SelectPartition(X)$}
\STATE{$\overline{\vartheta}_s \sim  UniformSample(\rho)$}
\STATE{$\ell, E_\ell \leftarrow Rollout(\mathcal{P},\pi(b,\overline{\vartheta}_s),H)$}
\STATE{$X \leftarrow (X\setminus \langle \rho,\hat{E}[\rho]\rangle) \cup \langle \rho\setminus I(\ell),\hat{E}[\rho]\rangle$}
\STATE{$X \leftarrow X \cup \langle \rho\cap I(\ell),\hat{E}[\rho]\cup E_\ell\rangle$}
\STATE{$X_{opt} \leftarrow \mathop{\arg \min}\limits_{\langle \rho,\hat{E}[\rho]\rangle\in X} \hat{E}[\rho]$}
\ENDWHILE
\RETURN{$X_{opt}$}
\end{algorithmic}
\end{algorithm}
    \end{minipage}
    \vspace{-5pt}
  \end{wrapfigure}

In this section, we present a novel algorithm for optimizing the parameters for a parameterized BSQ policy using the theory of braids developed above. The Partition Refinement Search (PRS) algorithm (Algo.\,\ref{alg:ips}) constructs the set of partitions using hierarchical partition selection and refinement, where a partition is selected to be refined, a leaf that can occur in that partition is sampled and evaluated, and the partitions are refined   to isolate the interval of the braid corresponding to the sample. The hypothesized optimal partition is tracked and returned as the final result after timeout.

PRS constructs the first parameter space interval as the domain of all possible parameter values (line 3). This is 
set as the initial hypothesized optimal partition (line 4).
In each iteration, a partition $\rho$ is selected using exploration-exploitation approaches discussed in Sec.\,\ref{sec:ips_variants} (lines 6). A leaf $\ell$ is sampled from $\rho$ by uniformly sampling parameter value $\overline{\vartheta}$ from $\rho$'s parameter intervals and performing rollouts from the initial belief state to a reachable leaf using the BSQ policy $\pi(b,\overline{\vartheta})$ (lines 7 and 8). 
The sampled leaf $\ell$ is used to refine partition $\rho$ using the insight braids cannot overlap (Sec.\,\ref{sec:piecewise_const}). If there exists a subinterval of $\rho$ where $\ell$ does not occur, a new partition for this subinterval is constructed containing $\rho$'s previous leaves and expected cost (line 9). The remaining portion of $\rho$, where $\ell$ can occur, is used to construct a partition with an updated expected cost representing $\rho$'s previous leaves and $\ell$ (line 10).  The hypothesized optimal partition is then updated (line 11).

  PRS converges to the true optimal BSQ policies in the limit:
  

\begin{restatable}{thm}{prcomplete}
Let $\pi(b,\overline{\Theta})$ be a parameterized BSQ policy, $\mathcal{P}$ be a gPOMDP, $b_0$ the initial belief state, and $H$ be the horizon. The likelihood of the Partition Refinement Search algorithm returning the optimal parameter interval converges to one in the limit of infinite samples.
\label{theorem:Pr_complete}
\end{restatable}
\textbf{Complexity analysis }
While the theoretical space and time complexity are linear in the number of leaves, due to PRS grouping leaves from the strategy tree (Def.\,\ref{def:strat_tree}), there is good reason to expect better performance in practice. As discussed in Sec. \ref{sec:strate_tree_features},  strategy trees can get pruned with the removal of branches and leaves that do not align with the user's requirements. For example, in the Spaceship Repair problem using the Fig.\,\ref{fig:spaceship_repair_combined} parameterized BSQ policy, a third of the possible leaves are pruned at a horizon of two, and the pruning becomes even more significant for longer horizons.
Additionally, empirical results suggest that
rules earlier in rule-observation trajectories are more important in dictating the partitions. Furthermore, selecting and refining partitions can be performed in parallel, further improving performance.

\section{Partition Selection Approaches}
\label{sec:ips_variants}

We explored multiple partition selection approaches with a multiprocessing version of PRS. Each approach used the same dynamic exploration rate $e_r$ that diminished over time. Each thread managed a subset of partitions $ X'\subseteq X$ and updated a global hypothetical optimal partition. Additionally, we warm start PRS by randomly selecting 20 points in the parameter search space and evaluating them 40 times to build an initial set of partitions. Also, partitions that have a lower expected cost than the hypothesize optimal are sampled up to 40 before updating the hypothesize optimal. In this paper, we focus on three selection approaches and discuss two others in the Appendix. 

\textbf{Epsilon Greedy (PRS-Epsilon)}\quad
We explore $e_r$ percent of the time by uniformly sampling $s\sim U_0^1$ and checking if $s\leq e_r$. If we are exploring, we uniformly at random select a partition from $X'$. Otherwise, the partition with minimum expected cost, $\mathop{\arg \min}_{\langle \rho,\hat{E}[\rho]\rangle\in X'} \hat{E}[\rho]$, is selected.

\textbf{Boltzmann Exploration (PRS-Bolt)}\quad Partitions are selected in a weighted random fashion with the probability of selecting partition $\rho$ as $\alpha \times exp(\hat{E}[\rho]/e_r)$ with $\alpha$ being the normalization factor.

\textbf{Local Thompson Sampling (PRS-Local)}\quad
Each thread treats the problem as a multi-armed bandit problem where the expected cost for the next sample from each partition is simulated using $\mathcal{N}(\mu_c,\sigma_c)$ with $\mu_c$ and $\sigma_c$ being the partition's mean and standard deviation, respectively. The partition with the lowest estimated expected cost is selected. 

\section{Empirical Results}
\label{sec:empirical}

We created an implementation of PRS and evaluated it on four challenging risk-averse problems.
Complete source code is available in the supplementary material.
We describe the problems and user preferences here; further details, including parameterized BSQ policy, can be found in the Appendix.


\textbf{Lane merger (LM)}\quad In this problem, an autonomous vehicle driving on a two-lane road must switch lanes safely before reaching a lane merger. However, there is currently a car in the other lane that the agent does not know the location or speed of. Switching lanes too close to this car risks a severe accident. The autonomous vehicle has a noisy detection system that returns whether a vehicle is located in regions around the car. 
The user's preference is:
\emph{If there's a high likelihood of safely switching lanes, do so. If there is a high likelihood of the other car being in close proximity and it is possible to slow down, slow down. Otherwise, keep going.}

\textbf{Spaceship repair (SR)}\quad
This is a modified version of the running example with parameterized BSQ policy Fig.\,\ref{fig:spaceship_repair_combined}(b). The robots start 7 steps and 5 steps away from the robot and ship repair stations, respectively. Additionally, the robot's sensor is 75\% accurate at detecting errors with the robot and only 55\% for the ship. With the short horizon $H=12$, this results in the parameter space being not convex with multiple local minimums with differing expected costs.

\textbf{Graph rock sample (GRS)}\quad
We modified the classic RockSample($n,k$) problem \citep{10.5555/1036843.1036906} by replacing the grid with a graph with waypoints where some waypoints contain rocks. Additionally, we introduced risk by causing the robot to break beyond repair if it samples a rock not worth sampling. We also categorized the rocks into types, and the rover's goal is to bring a sample of each type to the desired location if a safe rock for that type exists. This goal requires a longer horizon to reach compared to the other problems. The user's preference is:
\emph{Evaluating rocks of types not sampled in order $r_1,...,r_n$, if the rock has a high likelihood of being safe to sample, go and take a sample of it. Else, if the rock has a high likelihood of being safe to sample, get close enough and scan it. Otherwise, move towards the exit if no rocks are worth sampling or scanning.
}

\textbf{Store visit (SV)}\quad
This problem is based on the partially observable OpenStreetMap problem in \citet{liu2021leveraging}. A robot is located in a city where some locations are unsafe (e.g., construction, traffic), which can terminally damage the robot. The robot is initially uncertain of its location but it can scan its surroundings to determine its general location. The agent traverses the city and can visit the closest building. The goal is to visit a bank and then a store. This problem features a nuanced parameterized BSQ policy:
\emph{If you are significantly unsure of your current location, scan the area. If you have visited a bank, do the following to visit a store; otherwise, do it to visit a bank. If you are sufficiently sure the current location has the building you are looking for, visit it. Otherwise, move towards where you think that building is while avoiding unsafe locations. If all else fails, scan the current area.}


\subsection{Baselines}

We evaluated PRS against three different types of baselines.

\textbf{RCompliant}\quad
Select random parameter values uniformly at random from the parameter space to produce user-aligned policies.

\textbf{Hyperparamter optimization algorithms}\quad
To measure the benefits of PRS against existing hyperparameter optimization algorithms, we implemented both \textit{Nelder-Mead} \citep{nelder1965simplex} and \textit{Particle Swarm} \citep{kennedy1995particle}. The expected cost of parameter space point $\overline{\vartheta}$, for parameterized BSQ policy $\pi$, was computed by averaging 1,000 parallel runs of the policy $\pi(b;\overline{\vartheta})$. For Nelder-Mead optimization, we used a simplex that had vertices numbering one more than the number of parameters in the parameterized BSQ policy being optimized. We warm start by initially evaluating a 100 random points to construct the initial simplex using the best-performing points. For Particle Swarm optimization, 10 particles were used with the location and momentum of each particle clipped to the search space. The coefficients changed based on steps since the last improvement.

\textbf{Unconstrained POMDP solvers}\quad
To measure the differences between BSQ policies and unconstrained POMDP solvers, we implemented variations of our problems into POMDPX and solved them with DESPOT \citep{somani2013despot} and SARSOP \citep{kurniawati2009sarsop} for 1,000 evaluation runs. To measure whether an action-observation trajectory produced with these solvers aligns with the user's requirements, we check if there exist parameter values $\overline{\vartheta}$ where policy $\pi(b;\overline{\vartheta})$ could produce that trajectory. We use this to evaluate the solutions produced.

\subsection{Analysis of Results}





\begin{figure*}[t!]
    \centering
    \includegraphics[width=1.0\textwidth]{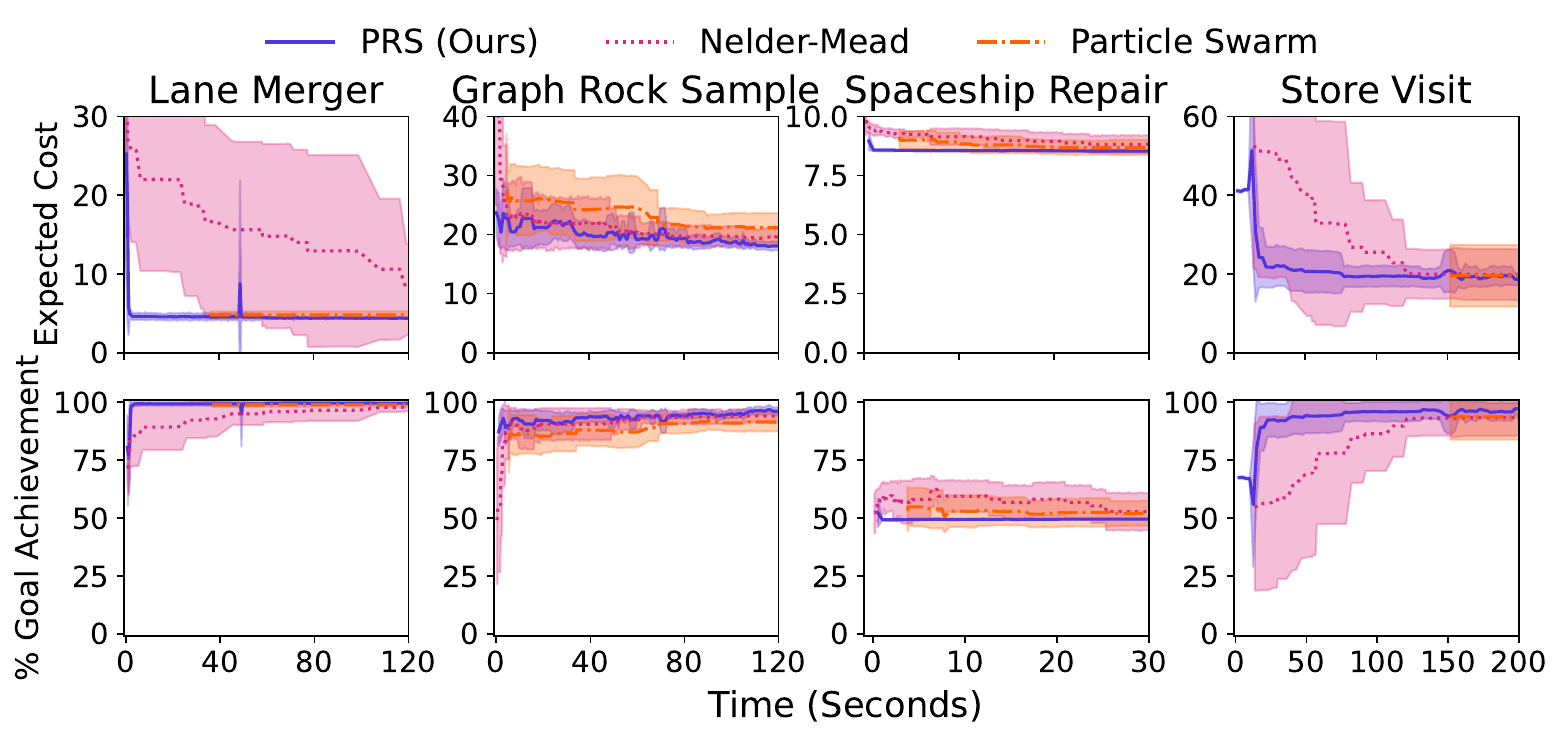}
    \caption{Empirical results evaluating the hypothesized optimal partition performance tracked. Equally spaced samples across PRS evaluation time are taken while a sample is taken each iteration of Nelder-Mead and Particle Swarm. The error displayed is the standard deviation error.}
    \label{fig:results}
\end{figure*}

For each problem, we evaluated each baseline and PRS variant ten times. The horizon was 12 for Spaceship Repair and 100 for the other problems. The timeout for PRS was set on a problem-by-problem basis. Timeout for Nelder-Mead and Particle Swarm was one hour. Note that the highest expected cost is equal to the horizon due to the default cost function.
The performance of each PRS partition selection approach can be found in Figure \ref{fig:prs_results} and the quality of solutions over time compared to the baselines are shown in Figure \ref{fig:results}.

\textbf{Partition selection approach evaluations}\quad
PRS partition selection approaches converged to a similar quality policy. The only difference was the time taken with approaches that did not rely on the standard deviation converging faster due to there being a lower standard deviation near the optimal solution, causing selection approaches that used the standard deviation to explore the wrong partitions. We use PRS-Bolt as a representative when comparing against the other baselines.

\begin{wrapfigure}{r}{7cm}
    \includegraphics[width=7cm]
    {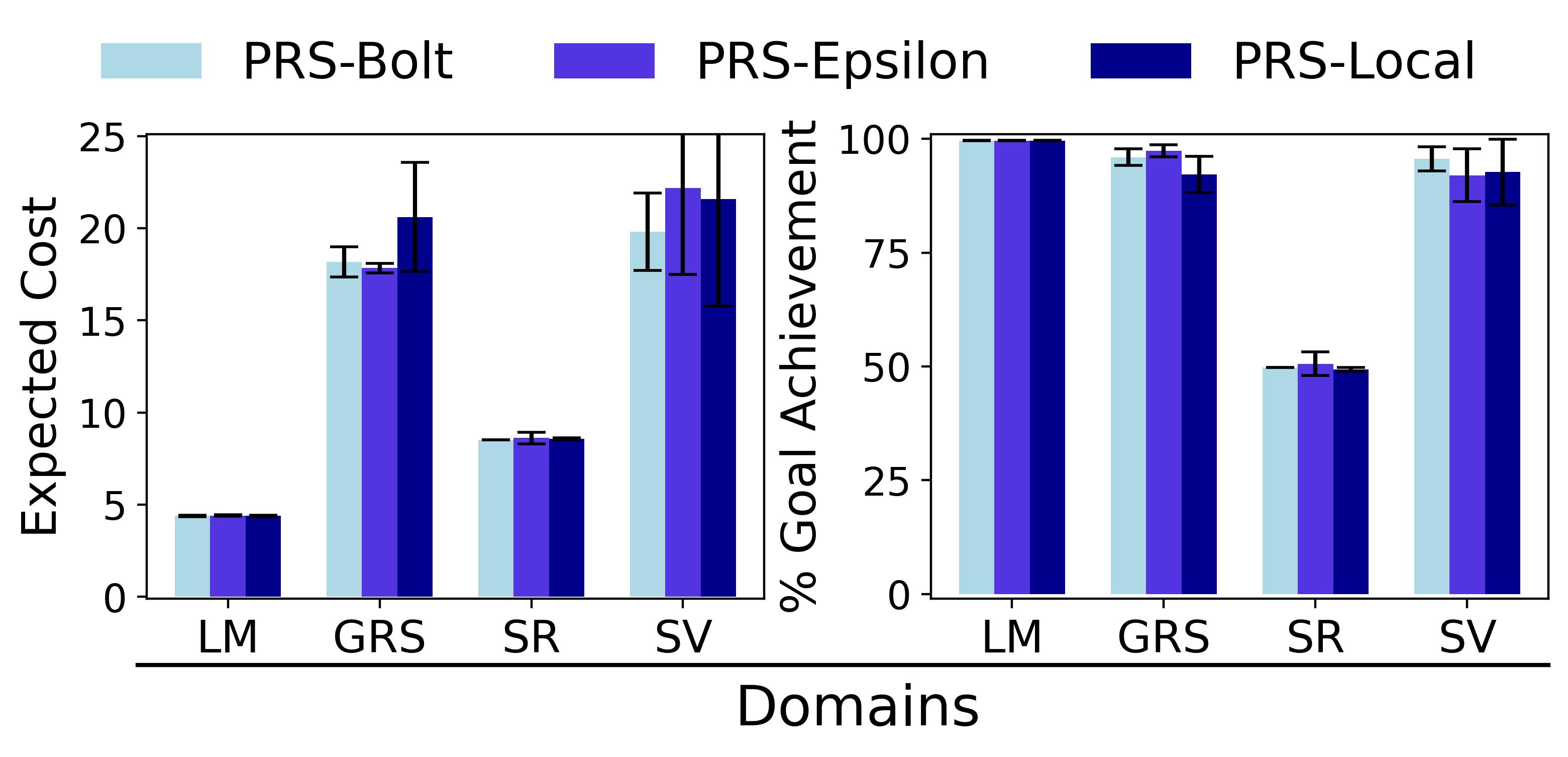}
    \caption{Results for PRS with different partition selection approaches from Section \ref{sec:ips_variants}.}
    \label{fig:prs_results}
\end{wrapfigure}

\textbf{PRS solution quality}\quad
PRS produced a higher-quality policy compared to the ones produced by RCompliant. For Spaceship Repair, the simplest problem solved on the shortest horizon, policies produced by PRS-Bolt had a 15.68\% lower expected cost and 3.47\% higher goal achievement rate. For the other problems, policies produced by RCompliant had more than triple the expected cost and achieved only half the success rate on both Graph Rock Sample and Store Visit. These results demonstrate that optimizing BSQ parameter values has a significant impact on the performance of user-aligned policies.


\textbf{Hyperparameter optimization evaluation}\quad 
Compared to traditional hyperparameter optimization algorithms, PRS always found the user-aligned policy with the lowest expected cost with little performance deviation. This is due to Nelder-Mead and Particle Swarm struggling to optimize a non-convex piecewise-constant function using noisy data, resulting in known problems with local-search algorithms: problems of getting stuck in sub-optimal local minima and exploring the incorrect space. Additionally, PRS converged first since it is more sample-efficient. It is computationally expensive to update the belief state, resulting in poor-quality solutions being more expensive to evaluate due to taking longer to reach the goal. PRS only requires a couple of evaluations before spending the computational resources on more promising areas. 

An interesting result is that, in Spaceship Repair, solutions found by Nelder-Mead and Particle Swarm both had a 7.73\% higher expected cost and 18.31\% higher goal achievement rate than the PRS-Epsilon solutions. There is likely a high negative correlation between the expected cost and goal achievement rate. PRS is better at optimizing the stated objective of minimizing the expected cost.

\textbf{Unconstrained solver evaluation}\quad
Without guiding from the parameterized BSQ policies, DESPOT and SARSOP struggled with this set of problems. SARSOP failed to converge to a policy due to the long problem horizon. DESPOT could not run on Lane Merger, which had the largest state space and branching factor. DESPOT also only achieved the goal 0.5\% of the time on Graph Rock Sample. DESPOT achieved a lower expected cost of 20.0\% and 13.3\% on variations of Spaceship Repair and Store Visit, respectively. However, DESPOT's policy never aligned with the user's requirements on Store Visit and only 7.3\%  of the time on Spaceship Repair. This indicates that the BSQ framework offers a new approach for expressing both domain knowledge and user requirements.

\section{Conclusion}
\label{sec:conclusion}

We presented the BSQ policy framework for expressing users' requirements over the belief state in partially observable settings for computing user-aligned agent behavior. We performed a formal analysis of these policies, proving that the parameter value space introduced in the parameterized BSQ policies can be partitioned, resulting in parameterized BSQ policies being optimizable through a hierarchical optimization paradigm. We introduced the probabilistically complete Partition Refinement Search algorithm to perform this optimization. Our empirical results show that it converges to the optimal user-aligned policy quicker and more consistently than existing approaches. Results indicate that parameterized BSQ policies provide a promising approach for solving diverse real-world problems requiring user alignment.

\textbf{Limitations and future work}\quad
There are many interesting directions for future work based on the current BSQ policy framework. BSQ representations can be made more expressive by allowing deterministic functions, which would not compromise the presented theoretical results. Furthermore, there exists a natural extension of this work into finite memory controllers that allows temporally extended requirements to be encoded with the same theoretical results. Relaxing the constraints on mapping each belief state to a single action would expand the usability. For more complex problems, a belief-state approximation approach would be required, but the underlying strategy tree discussed in this work would remain mostly unchanged. Another interesting research direction is to develop methods that help users express their requirements in the BSQ framework. 

\begin{ack}
This work was supported in part by ONR grant N000142312416 and NSF grant IIS 1942856.
\end{ack}

\bibliography{aaai24}

\appendix
\clearpage
\section{Appendix Organization}
The Appendix is organized as follows. Appendix \ref{appendix:pwc_lemmas_proofs} contains the proofs for showing the expected cost function is piecewise constant. Appendix \ref{appendix:prs_proof} contains the proof that PRS is probabilistically complete. Appendix \ref{appendix:problem_bsq_prefs} discusses the evaluation problems and provides the parameterized BSQ policies used. Appendix \ref{appendix:hop_implementation} discussed additional implementation details of both Nelder-Mead and Particle Swarm. Appendix \ref{appendix:additional_results} discusses two additional partition selection approaches we tested and provides additional analysis of our results. Appendix \ref{appendix:experimental_setup} discusses the experimental setup and computational cost of our experiments. Appendix \ref{appendix:sr_closed_form} contains the calculated closed-form solution of the partitions for the Spaceship Repair problem. Appendix \ref{append:broader_impacts} discusses the broader impacts of our work. Finally, Appendix \ref{appendix:limitations_future_work} discusses additional limitations not discussed in the main paper.

\section{Lemmas and Proofs From Formal Analysis [Section \ref{sec:formal}]}
\label{appendix:pwc_lemmas_proofs}

In this section, we provide the formal proofs for Lemma \ref{lemma:belief_param_eval}, Theorem \ref{theorem:equiv_partitions}, Theorem \ref{theorem:disjoint_int}, and Theorem \ref{theorem:piecewise_constant} from Section \ref{sec:formal}, where we proved that braids partition the parameter space resulting in the expected cost function of a parameterized BSQ policy w.r.t its parameter being piecewise constant. We define and prove Lemmas \ref{lemma_subsets}, \ref{lemma:no_overlap}, \ref{lemma:equiv_intervals}, and \ref{lemma:expected_value} in this section for building these proofs.

First, we prove that the similarity operator $\equiv_H$ for braids (Def.\,\ref{def:equiv}) has the properties of being reflexive, symmetric, and transitive. As such, $\equiv_H$ defines an equivalence relation over the n-dimensional parameter space $\mathbb{R}^n$, meaning it defines a partition over $\mathbb{R}^n$.

\equivpartitions*
\begin{proof}
Let $\overline{\vartheta} \in \mathbb{R}^n$ be $n$-parameter values and $H$ be the horizon.
By way of contradiction, let's assume that $\overline{\vartheta}$ is not similar to itself, $\overline{\vartheta}\not\equiv\overline{\vartheta}$. This would mean that $braid_{H,1}(\overline{\vartheta}) \neq braid_{H,2}(\overline{\vartheta})$. As such, there must exist a leaf $\ell$, which is in one but not the other braid. Note that $\ell$ represents a unique rule-observation trajectory $\{r_1,o_1,...,r_H,o_H\}$. Additionally, for $\ell$ to be in one of these braids it would need to be true that $\forall i, r_i.\Psi(b_i^*,\overline{\vartheta})$ must be satisfied, where $b_i^* = bp^*(b_0,r_1,o_1,...,r_i,o_i)$ (Def.\,\ref{def:braid}). However, note that this would hold true for the other braid as well, making it a contradiction for $\ell$ to be exclusive in either $braid_{H,1}(\overline{\vartheta})$ or $braid_{H,2}(\overline{\vartheta})$. As such, $\overline{\vartheta}$ must be similar to itself meaning the similarity property holds.

Let $\overline{\vartheta}_1,\overline{\vartheta}_2,\overline{\vartheta}_3 \in \mathbb{R}^n$ where  $\overline{\vartheta}_1 \equiv_H \overline{\vartheta}_2$ and $\overline{\vartheta}_2 \equiv_H \overline{\vartheta}_3$. Therefore, $braid_{\pi,H}(\overline{\vartheta}_1) = braid_{\pi,H}(\overline{\vartheta}_2)$ and $braid_{\pi,H}(\overline{\vartheta}_2) = braid_{\pi,H}(\overline{\vartheta}_3)$ (Def.\,\ref{def:braid}). Using substitution, $braid_{\pi,H}(\overline{\vartheta}_1) = braid_{\pi,H}(\overline{\vartheta}_3)$ meaning  $\overline{\vartheta}_1 \equiv_H \overline{\vartheta}_3$. As such, the transitive property holds.

Due to set equality being symmetric, the symmetric property holds. 
Thus, the operator $\equiv_H$ is an equivalence relation over $\mathbb{R}^n$ causing $\equiv_H$ to define a partition over $\mathbb{R}^n$.
\end{proof}

For compound BSQs $\Psi$, we now prove that there exist unique intervals of the parameter space where $\Psi$ is satisfied that we can calculate.

\beliefparameval*
\begin{proof}
    Let $\mathcal{P}$ be a gPOMDP, $b$ be a belief state, $\Theta\in\mathbb{R}$ be a parameter, $\circ$ be a comparison operator, and $\varphi$ be a first-order logic formula composed of functions from $\mathcal{P}$. There exist two possible forms for a BSQ (Def.\,\ref{def:bsq}). 
    Let  $\lambda_p(b;\varphi,\circ,\Theta)=Pr\llbracket\varphi\rrbracket_b \circ \Theta$. Note that $Pr\llbracket\varphi\rrbracket_b$ evaluates into the probability of $\varphi$ being satisfied in a belief state $b$. Therefore, we can simplify $\lambda_p(b;\varphi,\circ,\Theta)$ to $p\circ\Theta$ where $p\in[0,1]$, meaning this type of BSQ simplifies to an inequality.
    Now, let  $\lambda_p(b;\varphi,\circ,\Theta)=Pr\llbracket\varphi\rrbracket_b == 1$ where $\varphi$ is composed of $\Theta$ and fully observable functions in $\mathcal{P}$. We assume that $\Theta$ cannot be used as a function parameter, meaning that it must be an operand of a relational operator in $\varphi$. Since the functions are fully observable, they can be evaluated for $b$, leaving the inequalities involving $\Theta$ to dictate whether $\varphi$ is satisfied. Thereby, BSQs evaluate to inequalities involving $\Theta$. 
    
    A compound BSQ $\Psi$ comprises conjunctions/disjunctions of BSQs by Definition $\ref{def:comp_bsq}$. By substituting each BSQ with its inequalities, we can calculate the interval of $\Psi$, $I(\Psi)$.

Let us assume that $\Theta \in I(\Psi)$. By way of contradiction, let us assume that $\Theta$ does not satisfy $\Psi$. If $\Psi$ is a conjunction of BSQs, there exists at least one BSQ that is not satisfied by $\Theta$. If $\Psi$ is a disjunction, all the BSQs are unsatisfied by $\Theta$. However, this would mean that $\Theta$ cannot satisfy the inequalities from these BSQs, so $\Theta$ cannot be in $I(\Psi)$ since $I(\Psi)$ is constructed using the regions of the parameter space that satisfy the necessary BSQs, which is a contradiction.

Conversely, let us assume that $\Theta$ satisfies $\Psi$. This means one or all the BSQs are satisfied by $\Theta$ depending on if $\Psi$ is a conjunction or disjunction. If $\Theta$ was not in $I(\Psi)$, there could not exist a set of BSQs satisfied for $\Psi$ to be satisfied.

Thus, for a belief state $b$, a n-parameter compound BSQ $\Psi$ has an interval in the parameter space $I(\Psi)$ s.t. $\forall\Theta\in\mathbb{R}^n$, $\Theta\in I(\Psi)$ iff $\Theta(b;\Theta)$ evaluates true.
\end{proof}

As mentioned in Section \ref{sec:formal}, braids cannot be proper subsets of each other, which we will now prove in Lemma \ref{lemma_subsets}. As a high-level intuition, removing a leaf can only occur if a rule along that leaf's rule-observation trajectory is not satisfied, which would mean another rule must be satisfied since Def.\,\ref{def:bsq} guarantees coverage of the belief state and parameter space. This results in at least one leaf being added to a braid that removes this first leaf, making this new braid not a subset of the other one.

\begin{lemma}
Let $\pi(b,\overline{\Theta})$ be a parameterized BSQ policy, $\mathcal{P}$ be a gPOMDP, $b_0$ be the initial belief, and $H$ be the horizon. $\forall \overline{\vartheta}_1, \overline{\vartheta}_2 \in \mathbb{R}^n$, if $braid_{\pi,H}(\overline{\vartheta}_1) \subseteq braid_{\pi,H}(\overline{\vartheta}_2)$ then $braid_{\pi,H}(\overline{\vartheta}_1) = braid_{\pi,H}(\overline{\vartheta}_2)$.
\label{lemma_subsets}
\end{lemma}

\begin{proof}
Assume there exists $\overline{\vartheta}_1, \overline{\vartheta}_2 \in \mathbb{R}^n$ s.t. $braid_{\pi,H}(\overline{\vartheta}_1) \subset braid_{\pi,H}(\overline{\vartheta}_2)$ implying there exists leaf $\ell_{2}$ where $\ell_{2}\in braid_{\pi,H}(\overline{\vartheta}_2) \setminus braid_{\pi,H}(\overline{\vartheta}_1)$.

Let $\ell_{1} \in braid_{\pi,H}(\overline{\vartheta}_1)$ be the leaf with the largest rule-observation trajectory $\tau_0$ prefix shared with $\ell_{2}$ before differing. The trajectory for $\ell_{1}$ can be expressed as $\tau_0\tau_1$ where $\tau_1$ is the remaining trajectory for reaching $\ell_{1}$. Similarly, the trajectory for $\ell_{2}$ can be expressed as $\tau_0\tau_2$. 
Note $\tau_0$ represents the actions executed and observations observed from the initial belief state till right before the diversion resulting in the the belief state $b$ being the same for both leaves up to this point.

If the first element in $\tau_{1}$ and $\tau_{2}$ is a rule, note that $braid_{\pi,H}(\overline{\vartheta}_2)$ must also contain $\ell_{1}$. This would imply that $\pi(b;\overline{\vartheta}_2)$ is not mutually exclusive since two rules can occur in one element of the strategy tree. This is a contradiction by Def. \ref{def:bsq_pref}.
If the first element in $\tau_{1}$ and $\tau_{2}$ is an observation, different observations occurred after executing the last shared action in $\tau_{0}$. Due to the observation model and sharing the belief state $b$ at this point, both observations must be possible. This means a leaf in $braid_{\pi,H}(\overline{\vartheta}_1)$ must have a larger shared trajectory prefix than $\ell_{1}$, which is a contradiction. Thus, braids cannot be strict subsets of each other.
\end{proof}

Since braids cannot be proper subsets of each other, we can now prove that both braids must contain leaves the other does not have. In turn, this prevents the interval of braids from overlapping. Note that the interval of a braid can be calculated by taking the intersections of the intervals of each leaf contained in that braid (Def.\,\ref{def:leaf_interval}): $I(braid_{\pi,H}(\overline{\vartheta}))=\bigcap_{\ell \in braid_{\pi,H}(\overline{\vartheta})} I(\ell)$.

\begin{lemma}
Let $\pi(b,\overline{\Theta})$ be a parameterized BSQ policy,  $\mathcal{P}$ be gPOMDP, $b_0$ be the initial belief, and $H$ be the horizon. $\forall \overline{\vartheta}_1, \overline{\vartheta}_2 \in\mathbb{R}^n$, if $braid_{\pi,H}(\overline{\vartheta}_1) \cap braid_{\pi,H}(\overline{\vartheta}_2) \neq \varnothing$ and $braid_{\pi,H}(\overline{\vartheta}_1) \neq braid_{\pi,H}(\overline{\vartheta}_2)$ then $I(braid_{\pi,H}(\overline{\vartheta}_1)) \cap I(braid_{\pi,H}(\overline{\vartheta}_2)) = \varnothing$.
\label{lemma:no_overlap}
\end{lemma}

\begin{proof}
Let $\overline{\vartheta}_1, \overline{\vartheta}_2 \in\mathbb{R}^n$ where $braid_{\pi,H}(\overline{\vartheta}_1) \cap braid_{\pi,H}(\overline{\vartheta}_2) \neq \varnothing$ and $braid_{\pi,H}(\overline{\vartheta}_1) \neq braid_{\pi,H}(\overline{\vartheta}_2)$.
Both braids cannot be proper subsets (Lemma \ref{lemma_subsets}) meaning both braids must contain leaves that are not in the other braid: $braid_{\pi,H}(\overline{\vartheta}_2) \setminus braid_{\pi,H}(\overline{\vartheta}_1) \neq \varnothing$ and $braid_{\pi,H}(\overline{\vartheta}_1) \setminus braid_{\pi,H}(\overline{\vartheta}_2) \neq \varnothing$.

By Definition \ref{def:leaf_interval}, the interval of a braid is the conjunction of the intervals of each leaf it contains. Using the associative and commutative properties,
this can be rewritten as the conjunction of two sets: the interval of leaves shared and the interval of leaves not. 

$I(braid_{\pi,H}(\overline{\vartheta}_1)) = I(braid_{\pi,H}(\overline{\vartheta}_1) \cap braid_{\pi,H}(\overline{\vartheta}_2)) \cap I(braid_{\pi,H}(\overline{\vartheta}_1) \setminus braid_{\pi,H}(\overline{\vartheta}_2)$

$I(braid_{\pi,H}(\overline{\vartheta}_2)) = I(braid_{\pi,H}(\overline{\vartheta}_1) \cap braid_{\pi,H}(\overline{\vartheta}_2)) \cap I(braid_{\pi,H}(\overline{\vartheta}_2) \setminus braid_{\pi,H}(\overline{\vartheta}_1)$

 A braid's interval must exclude these unreachable leaves since a braid is all reachable leaves (Def.\,\ref{def:braid}). As such, $I(braid_{\pi,H}(\overline{\vartheta}_2))$ must not overlap with $I(braid_{\pi,H}(\overline{\vartheta}_1) \setminus braid_{\pi,H}(\overline{\vartheta}_2)$ and $I(braid_{\pi,H}(\overline{\vartheta}_1))$ must not overlap with $I(braid_{\pi,H}(\overline{\vartheta}_2) \setminus braid_{\pi,H}(\overline{\vartheta}_1)$. However, due to the conjunctions of intervals, $I(braid_{\pi,H}(\overline{\vartheta}_1) \subseteq I(braid_{\pi,H}(\overline{\vartheta}_1) \setminus braid_{\pi,H}(\overline{\vartheta}_2)$ and $I(braid_{\pi,H}(\overline{\vartheta}_2) \subseteq I(braid_{\pi,H}(\overline{\vartheta}_2) \setminus braid_{\pi,H}(\overline{\vartheta}_1)$. Thus, the intervals of $I(braid_{\pi,H}(\overline{\vartheta}_1)$ and $I(braid_{\pi,H}(\overline{\vartheta}_2)$ cannot overlap.
\end{proof}

The fact that two braids cannot have overlapping intervals allows us to prove that the sets of parameter values are similar iff they share the same braid interval.

\begin{lemma}
$\forall \overline{\vartheta}_1, \overline{\vartheta}_2 \in \mathbb{R}^n$, $\overline{\vartheta}_1 \equiv_H \overline{\vartheta}_2$ iff  $I(braid_{\pi,H}(\overline{\vartheta}_1)) = I(braid_{\pi,H}(\overline{\vartheta}_2))$.
\label{lemma:equiv_intervals}
\end{lemma}

\begin{proof}
Let $\overline{\vartheta}_1 \equiv_H \overline{\vartheta}_2$, meaning $braid_{\pi,H}(\overline{\vartheta}_1) = braid_{\pi,H}(\overline{\vartheta}_2) = L$  where L is the set of reachable leaves (Defs.\,\ref{def:braid} and \ref{def:equiv}). By Definition \ref{def:leaf_interval}, the interval of a set of leaves is the intersection of each leaf contained in the set, meaning both braids must have the same interval.

Let $I(braid_{\pi,H}(\overline{\vartheta}_1)) = I(braid_{\pi,H}(\overline{\vartheta}_2))$. By way of contradiction, assume $braid_{\pi,H}(\overline{\vartheta}_1) \neq braid_{\pi,H}(\overline{\vartheta}_2)$. By Lemma \ref{lemma:no_overlap}, this would mean $I(braid_{\pi,H}(\overline{\vartheta}_1)) \cap I(braid_{\pi,H}(\overline{\vartheta}_2)) = \varnothing$, which is a contradiction. Thus, $braid(\overline{\vartheta}_1) = braid(\overline{\vartheta}_2)$ meaning $\overline{\vartheta}_1 \equiv_H \overline{\vartheta}_2$ (Def.\,\ref{def:equiv}).
\end{proof}

We can now prove that partitions produced by $\equiv_H$ partitioning the parameter space $\mathbb{R}^n$ each represent a single braid, causing each partition to have a disjoint interval where a constant set of leaves is reachable.
\disjointint*
\begin{proof}
Let $\rho$ be a partition produced by $\equiv_H$ partitioning the parameter space $\mathbb{R}^n$. Note that this means that parameter value sets contained in $\rho$ must be similar (Def.\,\ref{def:equiv}): $\forall\overline{\vartheta}_1,\overline{\vartheta}_2\in \rho,\overline{\vartheta}_1\equiv_H\overline{\vartheta}_2$. As such, all parameter values have the same braid (Def.\,\ref{def:braid}), meaning there exists a set of leaves $L$ that are reachable in $\rho$. By Def.\,\ref{def:leaf_interval}, this set's interval must be $I(L) = \bigcap _{\ell_H \in L} I(\ell_H)$. By Lemma\,\ref{lemma:no_overlap}, the interval of other braids cannot overlap with $I(L)$. Also, there are no proper subsets (Lemma\,\ref{lemma_subsets}), meaning that no other braid can occur in $I(L)$ making it disjoint. 

By Def.\,\ref{def:leaf_interval}, $I(L)$ must be contained in $\rho$ due to all parameter value sets in $I(L)$ having the same braid of $L$ leaves. If $\rho$ contained parameter value sets not in $I(L)$, this would imply there exists $\overline{\vartheta}$ outside of $I(L)$ where just the leaves in $L$ are reachable, which is a contradiction due to $I(L)$ being the only interval space where all the leaves of $L$ are reachable. Meaning the interval of $\rho$ is actually $I(L)$. Thus, each partition represents a disjoint interval where only all leaves in $L$ are reachable.
\end{proof}

Due to the braid intervals not overlapping, we can prove that parameter value sets contained in that braid's interval must have a constant expected cost.

\begin{lemma}
Let $\pi(b,\overline{\Theta})$ be a parameterized BSQ policy and $H$ be the horizon. $\forall\overline{\vartheta}_1\in\mathbb{R}^n$, $\forall \overline{\vartheta}_2, \overline{\vartheta}_3 \in I(braid(\overline{\vartheta}_1))$, $E_\pi(\overline{\vartheta}_2;H) = E_\pi(\overline{\vartheta}_3;H)$.
\label{lemma:expected_value}
\end{lemma}

\begin{proof}
    
    Let $\overline{\vartheta}_2,\overline{\vartheta}_3 \in I(braid_{\pi,H}(\overline{\vartheta_1})$ where $\overline{\vartheta_1}\in\mathbb{R}^n$ is a tuple of $n$ parameters. Note that $braid(\overline{\vartheta}_1) = braid(\overline{\vartheta}_2) = braid(\overline{\vartheta}_3)$ due there being no strict subsets (Lemma \ref{lemma_subsets}) and leaves in $\overline{\vartheta}_2$ and $\overline{\vartheta}_3$ would have to be reachable in $\overline{\vartheta}_1$. 
Each braid represents a policy tree (Def.\,\ref{def:bsq_policy}), and the expected cost is based on the probability distribution of leaves in the braid. Since both $\overline{\vartheta}_2$ and $\overline{\vartheta}_3$ represent the same policy tree, they must have identical expected cost values.
\end{proof}

It is now trivial to show that each partition represents a disjoint interval of the parameter space where the expected cost is constant.

\piecewiseconstant*
\begin{proof}
Let $\rho$ be a partition created by partitioning $\mathbb{R}^n$ with $\equiv_H$. By Theorem\,\ref{theorem:disjoint_int}, all parameter value sets in the disjoint interval of $\rho$ must have the same braid. As such, by Lemma\,\ref{lemma:expected_value}, the expected cost is constant for all the parameter sets. Thus, the disjoint interval of each partition must have a constant expected cost.
\end{proof}

\section{Proofs For Partition Refinement Search [Section \ref{sec:algorithms}]}
\label{appendix:prs_proof}
In this section, we provide the formal proof for Theorem\,\ref{theorem:Pr_complete} proving that the Partition Refinement Search  algorithm introduced in Section\,\ref{sec:algorithms} is probabilistically complete. We define and prove Lemmas \ref{lemma:prs_not_empty}, \ref{lemma:prs_braid}, and \ref{lemma:prs_unique_braid} for building this proof.

When PRS refines a partition $\rho$ using a leaf $\ell$, it can produce up to two possible partitions: a partition for $\rho$ where $\ell$ is reachable and a partition for $\rho$ where $\ell$ is not (if it exists). We now show that this process prevents empty partitions.

\begin{lemma}
    \label{lemma:prs_not_empty}
    Let $\pi(b,\overline{\Theta})$ be a parameterized BSQ policy, $\mathcal{P}$ be a gPOMDP, $b_0$ be the initial belief state, and $H$ be the horizon. For each partition $\rho$ constructed by Partition Refinement Search, $\rho\neq\varnothing$. 
\end{lemma}
\begin{proof}
    Let $\rho\subseteq\mathbb{R}^n$ be a partition constructed by PRS. Since PRS creates partitions based on whether sampled leaves are included or excluded, let $L_i$ be the leaves PRS included in partition $\rho$ and $L_e$ be the leaves excluded. Therefore, $\rho=I(L_i)\setminus I(L_e)$.

    By way of contradiction, let $\rho=\varnothing$. There are two cases where this could occur: (1) excluding leaf $\ell$ caused $\rho=\varnothing$ or (2) including $\ell$ caused $\rho=\varnothing$. For case (1), we explicitly do not add partitions if excluding the leaf results in an empty interval, meaning this cannot happen. For case (2), this implies that there exists a previous partition $\rho_0$ where sampling leaf $\ell_0$ resulted in the partition constructed from $\rho_0$ including $\ell_0$ creating $\rho$ where $\rho=\varnothing$. Due to $\ell_0$ being uniformly sampled from $\rho_0$, $\ell_0$ must be reachable in $\rho_0$ meaning $\rho_0\cap I(\ell_0)\neq\varnothing$. However, line 10 of Algo.\,\ref{alg:ips} calculates the interval of $\rho$ as $\rho_0\cap I(\ell_0)$ meaning $\rho\neq\varnothing$, which is a contradiction. Thus, all partitions must be not empty.
\end{proof}

A critical property of PRS is that each partition constructed converges to represent a single braid.
\begin{lemma}
\label{lemma:prs_braid}
    Let $\pi(b,\overline{\Theta})$ be a parameterized BSQ policy, $\mathcal{P}$ be a gPOMDP, $b_0$ be the initial belief state, and $H$ be the horizon. Let $\rho$ be a partition constructed by Partition Refinement Search. If all leaves reachable in $\rho\subseteq\mathbb{R}^n$ have been sampled, $\forall \overline{\vartheta}\in \rho, I(braid_{\pi,H}(\overline{\vartheta})) = \rho$.    
\end{lemma}
\begin{proof}

Let $\rho\subseteq\mathbb{R}^n$ be a partition constructed by PRS. Let $L_\rho = \{\ell_1,...,\ell_n\}$ be the n-sampled unique leaves for $\rho$. Let all leaves reachable from $\rho$ be sampled, $\forall \ell, \ell \in L_\rho \leftrightarrow [\exists\overline{\vartheta}\in\rho,\ell\in braid_{\pi,H}(\overline{\vartheta})]$.

Due to $\rho\neq\varnothing$ (Lemma\,\ref{lemma:prs_not_empty}) and parameterized BSQ policies covering $\mathbb{R}^n$ (Def.\,\ref{def:bsq_pref}), there must exist a non-empty set of leaves $L$ reachable within $\rho$. Since all leaves are sampled, we know that $L\subseteq L_\rho$. However, there cannot be proper subsets (Lemma\,\ref{lemma_subsets}) meaning $L= L_\rho$. This means that the interval of $\rho$ must also equal the interval of leaves $I(L)$. Thus, $\rho$ must represent a braid.
\end{proof}

Since partitions are constructed by including/excluding sampled leaves hierarchically, we can prove that this makes each partition represent a unique braid.
\begin{lemma}
\label{lemma:prs_unique_braid}
    Let $\pi(b,\overline{\Theta})$ be a parameterized BSQ policy, $\mathcal{P}$ be a gPOMDP, $b_0$ be the initial belief state, and $H$ be the horizon. Let $\rho_1,\rho_2\subseteq\mathbb{R}^n$ be partitions constructed by Partition Refinement Search. If all leaves reachable in $\rho_1$ and $\rho_2$ have been sampled, $\forall \overline{\vartheta}_1\in \rho_1, \forall \overline{\vartheta}_2\in \rho_1, braid_{\pi,H}(\overline{\vartheta}_1) \neq braid_{\pi,H}(\overline{\vartheta}_2)$.    
\end{lemma}

\begin{proof}
Let $\rho_1,\rho_2\subseteq\mathbb{R}^n$ be two different partitions constructed by PRS. Note that the PRS partitions $\mathbb{R}^n$ by refining one partition using leaf $\ell$ into two by explicitly including $I(\ell)$ in one partition and excluding $I(\ell)$ in the other (Algorithm\,\ref{alg:ips}). Meaning $\rho_1$ and $\rho_2$ cannot overlap. 

Since both partitions represent a possible non-empty braid (Lemma\,\ref{lemma:prs_braid}), there exists a set of leaves reachable in both partitions. However, by the partition construction process, there must exist at least one leaf included in one but excluded in the other. Due to there being no interval overlap between braids, two different braids must be reachable in each partition (Lemma\,\ref{lemma_subsets}). Thus, all partitions must represent a unique braid.
\end{proof}

Using the property that each partition in PRS represents a unique braid, we can now prove that PRS is probabilistically complete.

\prcomplete*
\begin{proof}

    Note that gPOMDPs have a finite set of observations and finite horizon (Def.\,\ref{def:gpomdp}), and parameterized BSQ policies have a finite number of rules (Def.\,\ref{def:bsq_pref}). As such, there exists a finite number of unique rule-observation trajectories in the strategy tree (Def.\,\ref{def:strat_tree}). Therefore, there exists a finite number of leaves due to each leaf having a unique rule-observation trajectory. This results in there only being a finite set of braids being all possible combinations of reachable leaves (Def.\,\ref{def:braid}). Since each partition represents a unique braid (Lemma\,\ref{lemma:prs_unique_braid}), the number of partitions must be finite.

    Let $\rho\subseteq\mathbb{R}^n$ be a partition constructed by PRS that is not equivalent to a braid. By Lemma\,\ref{lemma:prs_braid}, this means there exists a leaf $\ell$ reachable in $\rho$ that has not been sampled yet. This also means there must exist a non-empty interval $\rho\cap I(\ell)$ where sampling from $\rho$ can reach $\ell$. Due to uniform sampling selecting parameter values when sampling a leaf for refining the partition (line 7 of Algorithm\,\ref{alg:ips}), the probability of selecting a parameter value that could sample $\ell$ can be calculated as $\frac{\rho\cap I(\ell)}{\rho}=Pr(I(\ell)|\rho)$.  

    Note that $\ell$ represents a unique rule-observation trajectory $\{r_1,o_1,...,r_H,o_H\}$. Note the probability of an observation $o$ being observed in belief state $b$ after action $a$ is executed is $Pr(o|b,a)=\sum_{s^\prime}[\Omega(s^\prime,a,o)\sum_s\mathcal{T}(s,a,s^\prime)b(s)]$. Meaning that the probability of reaching $\ell$ during rollout is $Pr(\ell)=\prod_iPr(o_i|b_i)$ where $b_i = bp^*(b_0,r_1,o_1,...,r_i,o_i)$. Since we know that $\ell$ is reachable, $Pr(\ell)>0$.
    
    We assume that partition selection approaches discussed in Section\,\ref{sec:ips_variants} have a non-zero probability of refining any partition. Let $Pr(\rho)$ be the probability of $\rho$ being selected. This means that in any refinement step, the probability of sampling leaf $\ell$ is $Pr(\ell)Pr(I(\ell)/\rho)Pr(\rho)$. Due to each probability being greater than zero, the probability of any non-sampled leaf being sampled must be greater than zero. Therefore, with enough refinement steps, all the leaves will be sampled since there is only a finite number of leaves. Thus, the set of partitions will be refined to the set of braids as the number of samples increases to infinite.

    Note that each partition represents a unique braid (Lemma\,\ref{lemma:prs_unique_braid}) with a set probability distribution of outcomes based on the reachable leaves. Due to a non-zero probability of refining a partition $Pr(\rho)$, the sampled expected cost of a partition will converge to the actual expected cost due to the law of large numbers.

    Therefore, within a finite number of samples, the partitions constructed by PRS will accurately represent the set of braids with an accurate representation of their expected costs. Thus, PRS will find the minimal expected cost partition as the number of samples increases to infinite.
\end{proof}

\section{Evaluation Problem's Belief-State Query Preferences}
\label{appendix:problem_bsq_prefs}
In this section, we provide the parameterized BSQ policies for the Lane Merger, Graph Rock Sample, and Store Visit problems discussed in Section\,\ref{sec:empirical}. To do this, we first describe the functions that compose each problem's states and actions. We use loops and quantifiers in the parameterized BSQ policies for clarity that can be unrolled on a problem-by-problem basis.

\subsection{Lane Merger}

The Lane Merger problem is that there are two lanes, and the agent must merge into the other lane within a certain distance. In this other lane, there is another car whose exact location and speed are unknown. Therefore, there exist two objects in the environment: the agent (agent) and the other car (other). For either object $o$, the location and speed are tracked using the unary integer functions $loc(o)$ and $speed(o)$. For actions, the agent can increase their speed ($speed\_up()$), decrease their speed ($slow\_down()$), remain in their current lane at their current speed ($keep\_speed()$), or attempt to merge lanes ($merge()$). Using these functions, the parameterized BSQ policy $\pi_{lm}(b;\Theta_1,\Theta_2)$ is formally defined as follows.

\begin{flalign*}
     &\pi_{lm}(b;\Theta_1,\Theta_2): \\ \nonumber
     &\text{\quad If } Pr \llbracket loc(agent) > loc(other)+ speed(other) + 2 \vee \\ \nonumber &\text{\quad\quad} loc(agent)+speed(agent)+2<loc(other)\rrbracket_b > \Theta_1 \rightarrow merge() \\ \nonumber
     &\text{\quad Else if } Pr \llbracket |loc(agent) - loc(other)| \leq 1 \rrbracket_b > \Theta_2 \wedge \\ \nonumber &\text{\quad\quad} Pr\llbracket speed(agent) > 0\rrbracket_b==1 \rightarrow slow\_down() \\ \nonumber
    &\text{\quad Else } keep\_speed()
\end{flalign*}

\subsection{Graph Rock Sample}
\label{appendix:grs}

The Graph Rock Sample problem is that there is a rover with pre-programmed waypoints, where some waypoints contain rocks. These rocks have been categorized into types, and whether it is safe for the rover to sample them is unknown. The objective of the rover is to sample each type with a safe rock before traversing to a dropoff location. The objects are the waypoints, including the rocks $\{r_1,...,r_n\}$ and the dropoff location (dropoff). The rover knows if or if it is not located at waypoint $w$ using the unary Boolean function $loc(w)$. The rover also knows whether it needs to sample rocks of type $t$ using the unary Boolean function $needed(t)$. For any rock $r$, the distance from the rover, whether the rock is type $t$, and if the rock is safe to sample are tracked using the unary double function $distance(r)$ and the Boolean functions $type(r,t)$, and $safe(r)$, respectively. The rover can move to neighboring waypoint $w$ ($move(w)$), sample rock $r$ at its current waypoint ($sample(r)$), and scan any rock $r$ ($scan(r)$). For clarity, we use the function $goto(w)$ to specify taking the edge that moves the rover closer to waypoint $w$. Using these functions, the parameterized BSQ policy $\pi_{grs}(b;\Theta_1,\Theta_2,\Theta_3)$ is formally defined as follows.

\begin{flalign*}
     &\pi_{grs}(b;\Theta_1,\Theta_2,\Theta_3): \\ \nonumber
     &\text{\quad For } r_c\in\{r_1,...,r_n\}:\\ \nonumber
    &\text{\quad\quad If } Pr \llbracket \exists t| type(r_c,t)\wedge needed(t)\wedge loc(r_c)\wedge safe(r_c)\rrbracket_b \geq \Theta_1 \rightarrow sample(r_c) \\ \nonumber
    &\text{\quad\quad Else if } Pr \llbracket \exists t| type(r_c,t)\wedge needed(t)\wedge \neg loc(r_c)\wedge safe(r_c)\rrbracket_b \geq \Theta_1 \rightarrow goto(r_c) \\ \nonumber
    &\text{\quad\quad Else if } Pr \llbracket \exists t| type(r_c,t)\wedge needed(t)\wedge safe(r_c)\rrbracket_b \geq \Theta_2 \wedge \\ \nonumber
    &\text{\quad\quad\quad} Pr \llbracket distance(r_c) \leq \Theta_3\rrbracket_b==1 \rightarrow scan(r_c) \\ \nonumber
    &\text{\quad\quad Else if } Pr \llbracket \exists t| type(r_c,t)\wedge needed(t)\wedge safe(r_c)\rrbracket_b \geq \Theta_2 \wedge \\ \nonumber
    &\text{\quad\quad\quad} Pr \llbracket distance(r_c) > \Theta_3\rrbracket_b==1 \rightarrow goto(r_c) \\ \nonumber
    &\text{\quad Else } goto(dropoff)
\end{flalign*}

\subsection{Store Visit}

The Store Visit problem involves an agent in a city with a grid-based layout. Some locations are unsafe, while others contain a bank or a store. The objective is for the agent to visit a bank safely and then a store. The objects are the agent, the set of stores $\{s_1,..., s_n\}$, and the set of banks $\{b_1,..., b_m\}$. Labeling functions $bank(o)$ and $store(o)$ check whether object $o$ is a bank or store, respectively. The ternary Boolean function keeps track of the current $(x,y)$ location of the object $o$, $loc(o,x,y)$. Similarly, whether location $(x,y)$ is safe is tracked by the binary Boolean function $is\_safe(x,y)$. Lastly, the state keeps track of whether the agent has visited a bank using the nullary Boolean function $vbank()$. The agent can move left ($left()$), right ($right()$), up ($up$), and down ($down()$) in the grid. The agent can also visit a building in its current location ($visit()$) or scan its surroundings to figure out its location ($scan()$). Using these functions, the parameterized BSQ policy $\pi_{sv}(b;\Theta_1,\Theta_2,\Theta_3)$ is formally defined as follows.

\begin{flalign*}
     &\pi_{sv}(b;\Theta_1,\Theta_2,\Theta_3): \\ \nonumber
     &\text{\quad If } \forall x,y|Pr \llbracket loc(agent,x,y)\rrbracket_b < \Theta_3 \rightarrow scan() \\ \nonumber
    &\text{\quad Else if } Pr \llbracket\exists s,x,y| vbank()\wedge store(s)\wedge loc(s,x,y)\wedge loc(agent,x,y)\rrbracket_b \geq \Theta_1 \rightarrow visit() \\ \nonumber
    &\text{\quad For } s_c\in\{s_1,...,s_n\}:\\ \nonumber
    &\text{\quad\quad Else if } Pr \llbracket\exists x_1,y_1,x_2,y_2| vbank()\wedge store(s_c)\wedge loc(agent,x_1,y_1)\wedge loc(s_c,x_2,y_2)\wedge \\ \nonumber 
    &\text{\quad\quad\quad} x_1 < x_2\wedge is\_safe(x_1+1,y_1)\rrbracket_b \geq \Theta_2 \rightarrow right() \\ \nonumber
    &\text{\quad\quad Else if } Pr \llbracket\exists x_1,y_1,x_2,y_2| vbank()\wedge store(s_c)\wedge loc(agent,x_1,y_1)\wedge loc(s_c,x_2,y_2)\wedge \\ \nonumber 
    &\text{\quad\quad\quad} x_1 > x_2\wedge is\_safe(x_1-1,y_1)\rrbracket_b \geq \Theta_2 \rightarrow left() \\ \nonumber
    &\text{\quad\quad Else if } Pr \llbracket\exists x_1,y_1,x_2,y_2| vbank()\wedge store(s_c)\wedge loc(agent,x_1,y_1)\wedge loc(s_c,x_2,y_2)\wedge \\ \nonumber 
    &\text{\quad\quad\quad} y_1 > y_2\wedge is\_safe(x_1,y_1-1)\rrbracket_b \geq \Theta_2 \rightarrow down() \\ \nonumber
    &\text{\quad\quad Else if } Pr \llbracket\exists x_1,y_1,x_2,y_2| vbank()\wedge store(s_c)\wedge loc(agent,x_1,y_1)\wedge loc(s_c,x_2,y_2)\wedge \\ \nonumber 
    &\text{\quad\quad\quad} y_1 < y_2\wedge is\_safe(x_1,y_1+1)\rrbracket_b \geq \Theta_2 \rightarrow up() \\ \nonumber
    &\text{\quad Else if } Pr \llbracket\exists k,x,y| \neg vbank()\wedge bank(k)\wedge loc(k,x,y)\wedge loc(agent,x,y)\rrbracket_b \geq \Theta_1 \rightarrow visit() \\ \nonumber
    &\text{\quad For } k_c\in\{k_1,...,k_m\}:\\ \nonumber
    &\text{\quad\quad Else if } Pr \llbracket\exists x_1,y_1,x_2,y_2| \neg vbank()\wedge bank(k_c)\wedge loc(agent,x_1,y_1)\wedge loc(k_c,x_2,y_2)\wedge \\ \nonumber 
    &\text{\quad\quad\quad} x_1 < x_2\wedge is\_safe(x_1+1,y_1)\rrbracket_b \geq \Theta_2 \rightarrow right() \\ \nonumber
    &\text{\quad\quad Else if } Pr \llbracket\exists x_1,y_1,x_2,y_2| \neg vbank()\wedge bank(k_c)\wedge loc(agent,x_1,y_1)\wedge loc(k_c,x_2,y_2)\wedge \\ \nonumber 
    &\text{\quad\quad\quad} x_1 > x_2\wedge is\_safe(x_1-1,y_1)\rrbracket_b \geq \Theta_2 \rightarrow left() \\ \nonumber
    &\text{\quad\quad Else if } Pr \llbracket\exists x_1,y_1,x_2,y_2| \neg vbank()\wedge bank(k_c)\wedge loc(agent,x_1,y_1)\wedge loc(k_c,x_2,y_2)\wedge \\ \nonumber 
    &\text{\quad\quad\quad} y_1 > y_2\wedge is\_safe(x_1,y_1-1)\rrbracket_b \geq \Theta_2 \rightarrow down() \\ \nonumber
    &\text{\quad\quad Else if } Pr \llbracket\exists x_1,y_1,x_2,y_2| \neg vbank()\wedge bank(k_c)\wedge loc(agent,x_1,y_1)\wedge loc(k_c,x_2,y_2)\wedge \\ \nonumber 
    &\text{\quad\quad\quad} y_1 < y_2\wedge is\_safe(x_1,y_1+1)\rrbracket_b \geq \Theta_2 \rightarrow up() \\ \nonumber
    &\text{\quad Else } scan()
\end{flalign*}

\section{Hyperparameter Optimization Algorithms Implementation}
\label{appendix:hop_implementation}

As a baseline comparison, we implemented Nelder-Mead and Particle Swarm as hyperparameter optimization algorithms to compare solving for the optimal parameter values for a parameterized BSQ policy to minimize the expected cost of the resulting BSQ policy. Both algorithms evaluate points in the parameter space to decide which areas to explore next. For both, we evaluate a parameter point by taking a thousand parallel runs of the BSQ policy with those values to approximate the expected cost.

\textbf{Nelder-Mead}\quad We used a simplex that has edges numbering one more than the number of parameters in the parameterized BSQ policy being optimized. To start with a better initial simplex, we randomly sampled a hundred points and tracked the points that had lower expected costs and were 0.4 distance away from each of the better-performing points. The closer points were saved but were given a lower priority. 
Each iteration followed the standard Nelder-Mead steps with the sum quality of all the edges in the simplex calculated. If five iterations pass without an increase in quality, the run is deemed to have converged, and the best quality point of the simplex is returned as the solution.

\textbf{Particle Swarm}\quad
Particle swarm used 10 particles randomly selected from within the parameter space with a random velocity. Let $t$ be the number of iteration steps since the last improvement in the best quality point found. For each iteration, the cognitive coefficient is $1.0-0.1t$, and the social coefficient is $0.1 + 0.1t$, which causes the particles to become more greedy as time since the last improvement increases. The momentum is statically set to 0.6 with the velocity clipped between $\pm0.5$. The location of points is also clipped to the parameter search space. If 10 iteration steps pass without seeing an improvement, the run is deemed to have converged, and the best quality point of the swarm is returned as the solution.

\section{Additional Results}
\label{appendix:additional_results}

In this section, we provide additional results from the experiments performed. This includes introducing two additional partition selection approaches we evaluated: Global Thompson Sampling and Maximum Confidence. We also provide graphs of the performance of the hypothesized optimal partition across all PRS variants. Finally, we provide a results table for all five partition selection approaches and the baseline RCompliant.

PRS is implemented for multiprocessing by having each process manage a subset of the partitions $X'\subseteq X$ but share a global hypothesis of the optimal partition. Also, a dynamic exploration rate $e_r$ is used that diminishes over the solving time. Using this framework, two additional partition selection approaches were explored.

\textbf{Maximum Confidence (PRS-Max)}\quad
We explore $e_r$ percent of the time by uniformly sampling $s\sim U_0^1$ and checking if $s\leq e_r$. If exploring, we uniformly at random select a partition from $X'$. Otherwise, the partition with maximum standard deviation, $\mathop{\arg \min}_{\langle \rho,\hat{E}[\rho]\rangle\in X'} \sigma(\hat{E})[\rho]$, is selected. 

\textbf{Global Thompson Sampling (PRS-Global)}\quad  Unlike the other partition selection approaches, each processor iterates over all partitions it manages before selecting multiple partitions to refine. Partitions are chosen for two reasons: (1) they are below the minimum number of samples, or (2) the partition has the potential of being better than the current global hypothesized optimal partition. This is simulated for each partition using $\mathcal{N}(\mu_c,\sigma_c \times e_r)$ with $\mu_c$ and $\sigma_c$ being the mean and standard deviation of that partition, respectively. If the sample taken from this normal distribution has a lower expected cost than the hypothesized optimal partition, this partition is selected for refinement.

\begin{figure*}
    \centering
    \includegraphics[width=1.0\textwidth]{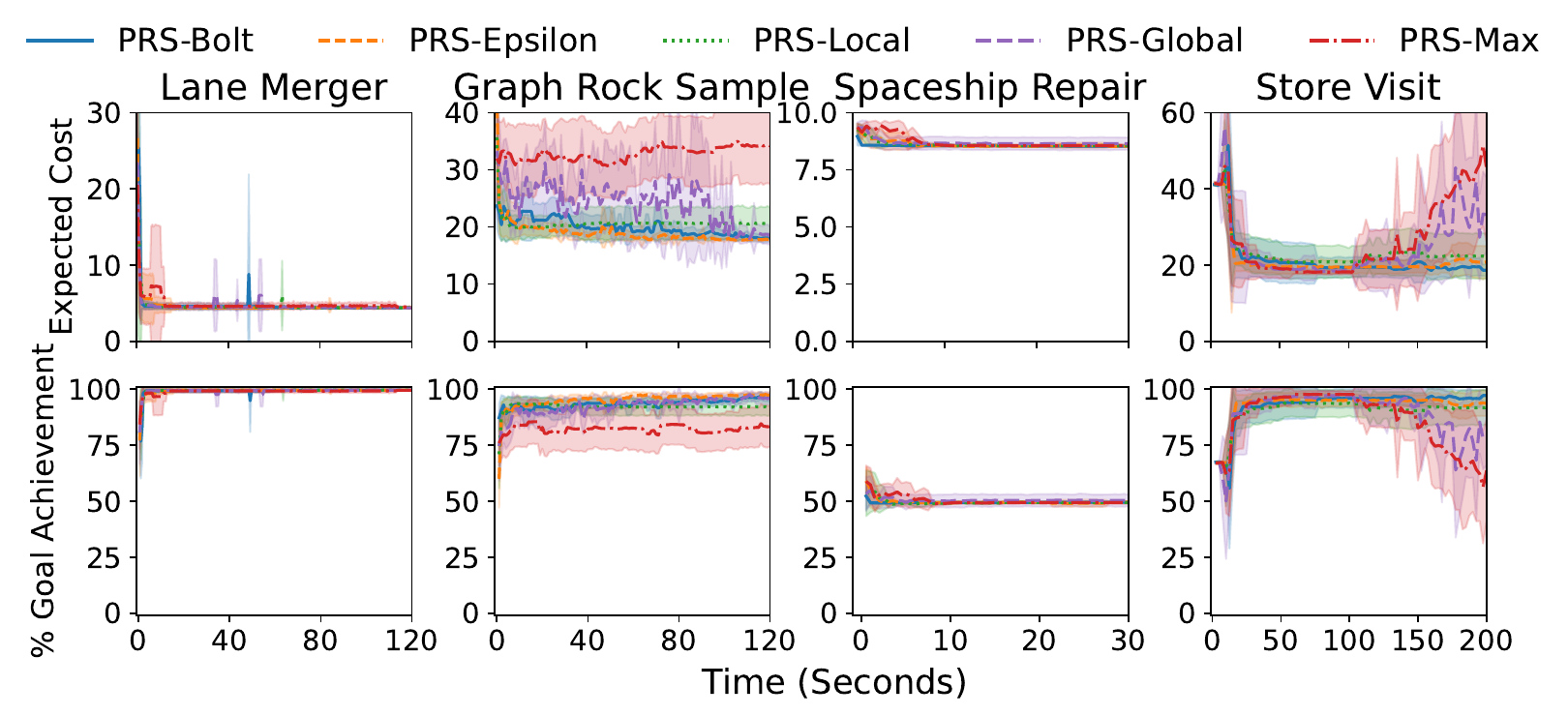}
    \caption{Performance of the hypothesized optimal partition while solving for the Lane Merger, Spaceship Repair, and Store Visit problems. Each line is the average over 10 independent runs with the standard deviation error shown.}
    \label{fig:time_graphs}
\end{figure*}

\textbf{Performance of the hypothesized optimal} In Figure\,\ref{fig:time_graphs}, the hypothesized optimal over the runtime of PRS for each partition selection approach is shown. On Lane Merger and Spaceship Repair, the performance of each PRS variant is quite similar, with the solver quickly converging to a near-optimal policy. However, PRS-Global has a much slower convergence rate due to trying to evaluate all promising partitions rather than focusing on the most promising ones. This resulted in PRS-Global not converging before timeout on Store Visit. PRS-Max is expected to perform poorly due to its poor partition-selection strategy. These results highlight that, with a competent partition-selection strategy, PRS will converge to the optimal policy that minimizes the expected cost, with the main variation being the convergence time.

\textbf{Tabulated performance} In Table\,\ref{tab:expected_cost_table} and Table\,\ref{tab:goal_rate_table}, the expected cost and the goal achievement rate have been tabulated, showing the near identical performance of four of the partition refinement approaches. The solution for Nelder-Mead and Particle Swarm are taken at PRS's timeout time to give each solver the same solving time. For all the problems, the more effective partition-selection approaches discussed in the main paper achieved equal, if not better, performance than Nelder-Mead and Particle Swarm. 

\begin{table}
	\centering
	\begin{tabular}{l||c|c|c|c}
		Problems & Lane Merger & Graph Rock Sample & Spaceship Repair & Store Visit \\\hline\hline
		PRS-Bolt & $4.39 \pm 0.04$ & $18.19 \pm 0.82$ & $8.52 \pm 0.00$ & $19.81 \pm 2.09$\\
		PRS-Epsilon & $4.40 \pm 0.04$ & $17.84 \pm 0.25$ & $8.52 \pm 0.00$ & $22.20 \pm 4.70$\\
		PRS-Global & $4.40 \pm 0.03$ & $18.47 \pm 1.72$ & $8.61 \pm 0.31$ & $38.07 \pm 13.63$\\
		PRS-Local & $4.39 \pm 0.03$ & $20.61 \pm 2.96$ & $8.57 \pm 0.05$ & $21.57 \pm 5.80$\\
		PRS-Max & $4.48 \pm 0.08$ & $34.12 \pm 7.00$ & $8.58 \pm 0.07$ & $58.16 \pm 18.33$\\\hline
		Nelder-Mead & $4.89 \pm 0.65$ & $19.56 \pm 1.44$ & $8.82 \pm 0.38$ & $22.39 \pm 8.24$\\
		Particle Swarm & $4.91 \pm 0.56$ & $21.12 \pm 2.34$ & $8.69 \pm 0.33$ & $18.95 \pm 2.42$\\\hline
		RCompliant & $22.10 \pm 15.70$ & $60.91 \pm 18.46$ & $9.97 \pm 0.77$ & $56.64 \pm 33.04$\\\hline
	\end{tabular}
	\caption{Expected cost of Partition Refinement Search, Nelder-Mead, Particle Swarm, and RCompliant on the Lane Merger, Graph Rock Sample, Spaceship Repair, and Store Visit problems. The performance was measured over ten runs to calculate the performance average and standard deviation.}
	\label{tab:expected_cost_table}
\end{table}

\begin{table}
	\centering
	\begin{tabular}{l||c|c|c|c}
		Problems & Lane Merger & Graph Rock Sample & Spaceship Repair & Store Visit \\\hline\hline
		PRS-Bolt & $99.6\% \pm 0.1\%$ & $96.0\% \pm 1.8\%$ & $49.8\% \pm 0.0\%$ & $95.6\% \pm 2.7\%$\\
		PRS-Epsilon & $99.6\% \pm 0.1\%$ & $97.3\% \pm 1.3\%$ & $49.8\% \pm 0.0\%$ & $92.0\% \pm 5.8\%$\\
		PRS-Global & $99.6\% \pm 0.0\%$ & $96.3\% \pm 2.9\%$ & $50.6\% \pm 2.6\%$ & $72.6\% \pm 16.3\%$\\
		PRS-Local & $99.6\% \pm 0.1\%$ & $92.2\% \pm 4.0\%$ & $49.3\% \pm 0.4\%$ & $92.7\% \pm 7.2\%$\\
		PRS-Max & $99.5\% \pm 0.2\%$ & $83.6\% \pm 9.6\%$ & $49.3\% \pm 0.6\%$ & $48.2\% \pm 21.3\%$\\\hline
		Nelder-Mead & $98.9\% \pm 1.2\%$ & $94.0\% \pm 2.7\%$ & $52.7\% \pm 7.9\%$ & $92.4\% \pm 10.6\%$\\
		Particle Swarm & $99.0\% \pm 1.1\%$ & $91.3\% \pm 3.5\%$ & $52.0\% \pm 5.2\%$ & $96.6\% \pm 3.5\%$\\\hline
		RCompliant & $86.9\% \pm 16\%$ & $41.3\% \pm 20.5\%$ & $48.1\% \pm 10.2\%$ & $49.8\% \pm 38.4\%$\\\hline
	\end{tabular}
	\caption{Goal achievement rate of Partition Refinement Search, Nelder-Mead, Particle Swarm, and RCompliant on the Lane Merger, Graph Rock Sample, Spaceship Repair, and Store Visit problems. The performance was measured over ten runs to calculate the performance average and standard deviation.}
	\label{tab:goal_rate_table}
\end{table}

\section{Experimental Setup And Computational Cost}
\label{appendix:experimental_setup}
In this section, we go through the empirical setup of the experiments performed in Section\,\ref{sec:empirical} and include an estimate of the computation cost for running the experiments for this paper.

\begin{table}
    \centering
    \begin{tabular}{c|c|c}
         Problem & Timeout (seconds) & Sample Rate (seconds) \\ \hline
         Lane Merger & 120 & 0.5 \\
         Graph Rock Sample & 120 & 1 \\
         Spaceship Repair & 30 & 0.125 \\
         Store Visit & 300 & 2.5 \\
    \end{tabular}
    \caption{The timeout and the sample rate of the hypothesized optimal partition for PRS for the evaluation problems.}
    \label{tab:problem_times}
\end{table}

All experiments were performed on an Intel(R) Xeon(R) W-2102 CPU @ 2.90GHz without using a GPU. The Partition Refinement Search algorithm was implemented using a manager-worker design pattern where 8 workers were initialized when solving. The manager maintained the hypothesized optimal partition and current exploration rate. Table\,\ref{tab:problem_times} shows the timeout and sample rate used for each problem for PRS. PRS was allowed to use an addition minute beyond timeout in the case the hypothesized optimal partition had less than the minimum allowed samples, however this case did not occur.

Both solutions and recorded hypothesized optimal partitions were evaluated using the same random seed to ensure that the same initial states were assessed. This evaluation process was carried out in parallel using a manager-worker design pattern with 16 workers. 25,000 independent runs were conducted for each solution to determine the expected cost and goal achievement rate. Additionally, for each recorded hypothesized optimal partition, 10,000 runs were performed. The average performance and standard deviation error were calculated by averaging the results of ten runs for each combination of problem and solver. A similar approach was used to evaluate the random-parameter user-compliant policy RCompliant. Instead of using solved policies, ten parameter value sets were uniformly selected randomly from the parameter space, and each set was evaluated for 25,000 runs. These results are presented in Figure\,\ref{fig:results}.

For constructing the Spaceship Repair heatmap  (Figure\,\ref{fig:spaceship_repair_combined}), all combinations of parameters $\Theta_1$ and $\Theta_2$ were evaluated with parameter values sampled from 0 to 1 with increments of 0.002. This produced 251,001 equally-spaced parameter values. Parameter values were evaluated on 300 runs with a horizon of 12 to calculate the expected cost.

\textbf{Computational cost}\quad
Running the Partition Refinement Search algorithm for the empirical evaluation section (Section\,\ref{sec:empirical}) involved nine processes running simultaneously for 25 minutes across ten trials for each of the five partition-selection approaches. This resulted in 1.58 hours of CPU usage when run in parallel, equivalent to 14.22 hours if executed sequentially. Evaluation complexities were significant, such as the variance in time per run and problem type. For instance, evaluating the solutions and hypothesized optimal partitions for the Lane Merger problem using 17 processes took approximately 48 hours in parallel. The overall CPU usage for the main experiment approximates to 360 hours (15 days) in parallel, translating to about 6,288 hours (262 days) if run sequentially. Additionally, constructing the Spaceship Repair heatmap (Figure,\ref{fig:spaceship_repair_combined}) required approximately 24 hours of CPU time using 11 processes. These experiments were conducted thrice, culminating in an estimated total computational cost of 2,160 hours (90 days) using an Intel(R) Xeon(R) W-2102 CPU @ 2.90GHz, or 19,656 hours (819 days) if operations were performed sequentially.  

\section{Spaceship Repair Partitions Closed Form}
\label{appendix:sr_closed_form}
In this section, we calculate the braids that partition the parameter space for the Spaceship Repair problem with the parameterized BSQ policy from Fig.\,\ref{fig:spaceship_repair_combined}. 

First, we give the exact observation model used. From Section\,\ref{sec:gPOMDP}, the Spaceship Repair state is composed of two functions: $broken(o)$ and $rlocation()$. This means each state is expressed as $\{broken(robot),broken(ship),rlocation()\}$. Additionally, the set of observations can be expressed as $\{obs\_err(robot),obs\_err(ship)\}$. Let $p_r$ and $p_s$ be the probability of the observation reflecting the actual state of the robot and spaceship, respectively. The probability of observation $o$ in state $s$ after action $a$ is executed is calculated as follows.

\begin{equation}
\begin{aligned}
    &Pr(o=\{obs\_err(robot),obs\_err(ship)\}|s=\{broken(robot),broken(ship),rlocation()\}) = \\ 
    &\begin{cases}
        p_rp_s,& \text{if } broken(robot)=obs\_err(robot) \wedge broken(ship)=obs\_error(ship)\\
        p_r(1-p_s),& \text{if } broken(robot)=obs\_err(robot) \wedge broken(ship)\neq obs\_error(ship)\\
        (1-p_r)p_s,& \text{if } broken(robot)\neq obs\_err(robot) \wedge broken(ship)= obs\_error(ship)\\
        (1-p_r)(1-p_s),& \text{otherwise}
    \end{cases}    
\end{aligned}
\label{eq:spaceship_repair_obs_model}
\end{equation}

Note observations are independent of the robot's location and actions. For clarity, we express the states as whether or not the robot and spaceship are broken, $\{broken(robot),broken(ship)\}$. This means there are four possible states depending on whether the robot and ship are broken. For ease of notation, we represent these states as $S=\{s_{TT},s_{TF},s_{FT},s_{FF}\}$, where $s_{TF}$ represents that state where the robot is broken and the spaceship is not. Similar, let the four possible observations be represented as $O=\{o_{TT},o_{TF},o_{FT},o_{FF}\}$. 

The precondition of the first rule of the Spaceship Repair problem parameterized BSQ policy is $\llbracket broken(robot)\rrbracket_b \leq \Theta_1$ (Figure\,\ref{fig:spaceship_repair_combined}). For any belief state $b$, the probability of the robot being broken is the probability of the states where that is true: $\llbracket broken(robot)\rrbracket_b=b(s_{TT})+b(s_{TF})$.

Let $\{a_1,o_1,...,a_t,o_t\}$ be an action-observation trajectory for $t$ timesteps where at each timestep an action is executed followed by an observation being observed. We can calculate the probability of the state where the robot and spaceship are broken, $s_{TT}$, as follows.

\begin{equation}
    b_t(s_{TT})=\alpha Pr(o_t|s_{TT},a_t)\sum_s\mathcal{T}(s,a_t,s_{TT})b_{t-1}(s)
    \label{eq:state_prob}
\end{equation}

Note that, due to the observations being independent of the robot's location, the observation and transition functions are independent of the action. For example, there is no action the robot can perform to change whether the robot or spaceship is broken due to the problem being to reach a broken component rather than fixing it. We can simplify Equation\,\ref{eq:state_prob} significantly as follows.

\begin{equation}
    b_t(s_{TT})= \alpha Pr(o_t|s_{TT})b_{t-1}(s_{TT})
    \label{eq:state_prob_simp}
\end{equation}

We can now rewrite Equation\,\ref{eq:state_prob_simp} by unrolling the recursion. Note that $\alpha$ is the normalization factor meaning we don't need to calculate $\alpha$ each timestep because the final normalization will factor in all these changes. Additionally, due to the probability of each initial state being uniform, we don't need to keep track of the initial belief. Also, note there exist four possible observations. Due to the commutativity of multiplication, we can rearrange to get the following. Let $c_{TT},c_{TF},c_{FT},$ and $c_{FF}$ be the counts of the number of each observation where $c_{TT}+c_{TF}+c_{FT}+c_{FF}=t$.

\begin{equation}
    b_t(s_{TT})= \alpha Pr(o_{TT}|s_{TT})^{c_{TT}}Pr(o_{TF}|s_{TT})^{c_{TF}}Pr(o_{FT}|s_{TT})^{c_{FT}}Pr(o_{FF}|s_{TT})^{c_{FF}}
    \label{eq:state_prob_unroll}
\end{equation}

Using Equation \ref{eq:spaceship_repair_obs_model}, the probability of this state can be written in terms of $p_r$ and $p_s$.

\begin{equation}
    b_t(s_{TT})= \alpha (p_rp_s)^{c_{TT}}(p_r(1-p_s))^{c_{TF}}((1-p_r)p_s)^{c_{FT}}((1-p_r)(1-p_s))^{c_{FF}}
    \label{eq:stt_prob_unfac}
\end{equation}

\begin{equation}
    b_t(s_{TT})= \alpha p_r^{c_{TT}+c_{TF}}p_s^{c_{TT}+c_{FT}}(1-p_r)^{c_{FT}+c_{FF}}(1-p_s)^{c_{TF}+c_{FF}}
    \label{eq:stt_prob}
\end{equation}

This same process can be applied to the other three states to get the equation of their likelihoods.

\begin{equation}
    b_t(s_{TF})= \alpha p_r^{c_{TT}+c_{TF}}p_s^{c_{TF}+c_{FF}}(1-p_r)^{c_{FT}+c_{FF}}(1-p_s)^{c_{TT}+c_{FT}}
    \label{eq:stf_prob}
\end{equation}
\begin{equation}
    b_t(s_{FT})= \alpha p_r^{c_{FT}+c_{FF}}p_s^{c_{TT}+c_{FT}}(1-p_r)^{c_{TT}+c_{TF}}(1-p_s)^{c_{TF}+c_{FF}}
    \label{eq:sft_prob}
\end{equation}
\begin{equation}
    b_t(s_{FF})= \alpha p_r^{c_{FT}+c_{FF}}p_s^{c_{TF}+c_{FF}}(1-p_r)^{c_{TT}+c_{TF}}(1-p_s)^{c_{TT}+c_{FF}}
    \label{eq:sff_prob}
\end{equation}

We can group the states into two groups depending on whether or not the robot is broken. By factoring we can get the following.
\begin{equation}
\begin{aligned}
    b_t&(s_{TT}) + b_t(s_{TF})= \\
    &\alpha p_r^{c_{TT}+c_{TF}}(1-p_r)^{c_{FT}+c_{FF}}[p_s^{c_{TT}+c_{FT}}(1-p_s)^{c_{TF}+c_{FF}} + p_s^{c_{TF}+c_{FF}}(1-p_s)^{c_{TT}+c_{FT}}]
    \label{eq:rob_broken}
\end{aligned}
\end{equation}
\begin{equation}
\begin{aligned}
    b_t&(s_{FT}) + b_t(s_{FF})= \\
    &\alpha p_r^{c_{FT}+c_{FF}}(1-p_r)^{c_{TT}+c_{TF}}[p_s^{c_{TT}+c_{FT}}(1-p_s)^{c_{TF}+c_{FF}} + p_s^{c_{TF}+c_{FF}}(1-p_s)^{c_{TT}+c_{FT}}]
    \label{eq:rob_not_broken}
\end{aligned}
\end{equation}

Note that in Equations \ref{eq:rob_broken} and \ref{eq:rob_not_broken} everything in the brackets is shared, which is due to the individual observations of the spaceship and robot being independent of each other. Also, for normalization, we just divide the sum of Equations \ref{eq:rob_broken} and \ref{eq:rob_not_broken}, which is equivalent to the sum probability of all states. Additionally, note that Equation \ref{eq:rob_broken} is equivalent to the BSQ precondition $\llbracket broken(robot)\rrbracket_{b_t}$. Substituting into this BSQ and simplifying we get the following.
\begin{equation}
    \llbracket broken(robot)\rrbracket_{b_t}= \frac{p_r^{c_{TT}+c_{TF}}(1-p_r)^{c_{FT}+c_{FF}}}{p_r^{c_{TT}+c_{TF}}(1-p_r)^{c_{FT}+c_{FF}} + p_r^{c_{FT}+c_{FF}}(1-p_r)^{c_{TT}+c_{TF}}}
    \label{eq:rob_bsq}
\end{equation}

Note that there are two exponent values: the number of times the robot is observed to be broken and the number it is not. Let $d_r=c_{TT}+c_{TF}-c_{FT}-c_{FF}$ be the difference in the number of times that the robot is observed to be broken to not. If $d_r>0$, then the robot has been observed to be broken more often than not. By substituting $d_r+c_{FT}c_{FF}=c_{TT}+c_{TF}$ into Equation\,\ref{eq:rob_bsq} the equation simplifies down.
\begin{equation}
    \llbracket broken(robot)\rrbracket_{b_t}= \frac{p_r^{d_r}}{p_r^{d_r} + (1-p_r)^{d_r}}
    \label{eq:rob_bsq_simp}
\end{equation}

Following a similar process, the BSQ from the second rule in the parameterized BSQ policy from Figure\,\ref{fig:spaceship_repair_combined} can be written similarly. Let $d_s$ be the difference in the number of times the spaceship is observed to be or not. If $d_s>0$, then the spaceship has been observed to be broken more often than not.
\begin{equation}
    \llbracket broken(ship)\rrbracket_{b_t}= \frac{p_s^{d_s}}{p_s^{d_s} + (1-p_s)^{d_s}}
    \label{eq:ship_bsq_simp}
\end{equation}

Note that the observation model used $p_r=0.6$ and $p_s=0.75$ for the heatmap in Figure\,\ref{fig:spaceship_repair_combined}. The horizontal thresholds can be calculated using Equation\,\ref{eq:ship_bsq_simp} and the vertical with Equation\,\ref{eq:rob_bsq_simp}. 

Therefore, a partition is specific value of $d_r \in\mathbb{Z}$ and $d_s\in\mathbb{Z}$ that is equivalent to saying:
\emph{The objective is to fix the communication channel. If the difference in the number of times the robot has been observed being broken than not is greater than $d_r$, it should try to repair itself;
otherwise, if the difference in the number of times the spaceship has been observed being broken than not is greater than $d_s$, it should try to repair that.} Formally, this partition represents the parameter space where $\frac{p_r^{d_r-1}}{p_r^{d_r-1} + (1-p_r)^{d_r-1}} \leq \Theta_1 < \frac{p_r^{d_r}}{p_r^{d_r} + (1-p_r)^{d_r}}$ and $\frac{p_s^{d_s-1}}{p_s^{d_s-1} + (1-p_s)^{d_s-1}} \leq \Theta_2 < \frac{p_s^{d_s}}{p_s^{d_s} + (1-p_s)^{d_s}}$ where all parameter value sets that satisfy both inequalities are similar. Due to $d_r$ and $d_s$ being the difference between observation counts, the set of possible partitions is finite for finite horizons.

We explored solving the Spaceship Repair problem directly using these inequalities. The belief state reflects the probability of each outcome, meaning the main challenge is calculating the average number of timesteps to reach the goal. This can be solved by finding the average length of time of the Gambler's Ruin problem. One possible direction of future work is exploring solving parameterized BSQ policies and gPOMDPs this way.

\section{Broader Impacts}
\label{append:broader_impacts}

The primary positive impact of parameterized BSQ policies is their accessibility to non-experts, allowing them to input their requirements directly into a solver that optimizes the completion of tasks while aligning with the user. Moreover, parameterized BSQ policies enable encoding safety constraints with enforceable guarantees over the belief state. Thus, this paper represents an important step in making AI more usable for non-experts, particularly in encoding constraints and preferences, while addressing safety concerns in real-world applications."

A potential negative impact of making AI more accessible through parameterized BSQ policies is that it could also be exploited by bad actors who might encode harmful preferences. To mitigate this risk, one approach is to design goals such that negative outcomes inherently prevent goal completion, thereby teaching the agent to avoid these outcomes. Additionally, future work can explore methods for prioritizing certain constraints to ensure that the AI does not align with harmful intentions. 

\section{Additional Limitations}
\label{appendix:limitations_future_work}

While we discussed in Section\,\ref{sec:conclusion} some of the limitations of this work, one additional limitation is an essential direction of future work: aligning user and problem objectives. For example, in Graph Rock Sample, if the encoded goal for the gPOMDP did not require collecting rocks but the user still wanted to collect one rock of each type using the parameterized BSQ policy in Appendix\,\ref{appendix:grs}, PRS would optimize the parameters to make it so no rocks are worth scanning or sampling to exit as fast as possible to minimize the expected cost. While this case is an obvious misalignment between the gPOMDP and parameterized BSQ policy, these misalignments can be more subtle, leading to the optimal policy not behaving as intended. Therefore, future work needs to be done to explore catching misalignments to allow the user to understand and fix them.

\newpage
\section*{NeurIPS Paper Checklist}

\begin{enumerate}

\item {\bf Claims}
    \item[] Question: Do the main claims made in the abstract and introduction accurately reflect the paper's contributions and scope?
    \item[] Answer: \answerYes{}
    \item[] Justification: The claims made in both the abstract and introduction reflect the paper where we introduced a new framework for user preferences (Section \ref{sec:formal}), performed a formal analysis of it (Section \ref{sec:theory}), introduced a piecewise constant algorithm (Section \ref{sec:algorithms}), and empirically evaluated this algorithm (Section \ref{sec:empirical}). 

\item {\bf Limitations}
    \item[] Question: Does the paper discuss the limitations of the work performed by the authors?
    \item[] Answer: \answerYes{}
    \item[] Justification: Please refer to both Section \ref{sec:conclusion} and Appendix \ref{appendix:limitations_future_work}.

\item {\bf Theory Assumptions and Proofs}
    \item[] Question: For each theoretical result, does the paper provide the full set of assumptions and a complete (and correct) proof?
    \item[] Answer: \answerYes{}
    \item[] Justification: Refer to Appendix \ref{appendix:pwc_lemmas_proofs} and Appendix \ref{appendix:prs_proof} for the formal proofs of the lemmas and theorems defined in the paper.

    \item {\bf Experimental Result Reproducibility}
    \item[] Question: Does the paper fully disclose all the information needed to reproduce the main experimental results of the paper to the extent that it affects the main claims and/or conclusions of the paper (regardless of whether the code and data are provided or not)?
    \item[] Answer: \answerYes{}
    \item[] Justification: Both the code has been provided in the supplementary material and the detailed methodology can be found in Appendix \ref{appendix:experimental_setup}.

\item {\bf Open access to data and code}
    \item[] Question: Does the paper provide open access to the data and code, with sufficient instructions to faithfully reproduce the main experimental results, as described in supplemental material?
    \item[] Answer: \answerYes{}
    \item[] Justification: The code used for this paper has been provided in the supplementary material.

\item {\bf Experimental Setting/Details}
    \item[] Question: Does the paper specify all the training and test details (e.g., data splits, hyperparameters, how they were chosen, type of optimizer, etc.) necessary to understand the results?
    \item[] Answer: \answerYes{}
    \item[] Justification: Refer to Appendix \ref{appendix:experimental_setup} for a detailed methodology.

\item {\bf Experiment Statistical Significance}
    \item[] Question: Does the paper report error bars suitably and correctly defined or other appropriate information about the statistical significance of the experiments?
    \item[] Answer: \answerYes{}
    \item[] Justification: In both Section \ref{sec:empirical} and Appendix \ref{appendix:additional_results}, all shown results are shown with standard deviation error. Additionally, we make it clear in both sections that we are using standard deviation error.

\item {\bf Experiments Compute Resources}
    \item[] Question: For each experiment, does the paper provide sufficient information on the computer resources (type of compute workers, memory, time of execution) needed to reproduce the experiments?
    \item[] Answer: \answerYes{}
    \item[] Justification: Refer to Appendix \ref{appendix:experimental_setup}.
    
\item {\bf Code Of Ethics}
    \item[] Question: Does the research conducted in the paper conform, in every respect, with the NeurIPS Code of Ethics \url{https://neurips.cc/public/EthicsGuidelines}?
    \item[] Answer: \answerYes{} 
    \item[] Justification: We have reviewed and can confirm our research conforms to NeurIPS Code of Ethics.

\item {\bf Broader Impacts}
    \item[] Question: Does the paper discuss both potential positive societal impacts and negative societal impacts of the work performed?
    \item[] Answer: \answerYes{}
    \item[] Justification: Refer to Appendix \ref{append:broader_impacts}.
    
\item {\bf Safeguards}
    \item[] Question: Does the paper describe safeguards that have been put in place for responsible release of data or models that have a high risk for misuse (e.g., pretrained language models, image generators, or scraped datasets)?
    \item[] Answer: \answerNA{}.
    \item[] Justification: This paper poses no risk of being misused.

\item {\bf Licenses for existing assets}
    \item[] Question: Are the creators or original owners of assets (e.g., code, data, models), used in the paper, properly credited and are the license and terms of use explicitly mentioned and properly respected?
    \item[] Answer: \answerNA{}
    \item[] Justification: This paper does not use existing assets.

\item {\bf New Assets}
    \item[] Question: Are new assets introduced in the paper well documented and is the documentation provided alongside the assets?
    \item[] Answer: \answerYes{}
    \item[] Justification: Documentation on the code used in this paper is provided with the code.

\item {\bf Crowdsourcing and Research with Human Subjects}
    \item[] Question: For crowdsourcing experiments and research with human subjects, does the paper include the full text of instructions given to participants and screenshots, if applicable, as well as details about compensation (if any)? 
    \item[] Answer: \answerNA{}
    \item[] Justification: This paper does not involve crowdsourcing nor research with human subjects.

\item {\bf Institutional Review Board (IRB) Approvals or Equivalent for Research with Human Subjects}
    \item[] Question: Does the paper describe potential risks incurred by study participants, whether such risks were disclosed to the subjects, and whether Institutional Review Board (IRB) approvals (or an equivalent approval/review based on the requirements of your country or institution) were obtained?
    \item[] Answer: \answerNA{}
    \item[] Justification: This paper does not involve crowdsourcing nor research with human subjects.

\end{enumerate}

\end{document}